\documentclass[aos]{imsart}
\RequirePackage{amsthm,amsmath,amsfonts,amssymb}
\RequirePackage[colorlinks,citecolor=blue,urlcolor=blue]{hyperref}
\RequirePackage{graphicx,color,xcolor,comment,float}
\RequirePackage{caption,subcaption,booktabs,multirow,array}

\usepackage{natbib}
\usepackage{bbding}

\usepackage{hyperref}       % hyperlinks
\usepackage{url}            % simple URL typesetting
\usepackage{booktabs}       % professional-quality tables
\usepackage{amsfonts}       % blackboard math symbols
\usepackage{nicefrac}       % compact symbols for 1/2, etc.
\usepackage{microtype}      % microtypography
\usepackage{xcolor}         % colors

\usepackage{amssymb}
\usepackage{amsfonts, mathrsfs}
\usepackage{bm,float}
\usepackage{multirow}
\usepackage{xcolor}
\usepackage{pifont} %to define checkmark
\newcommand{\cmark}{\ding{51}}%
\newcommand{\xmark}{\ding{55}}
\usepackage{makecell}
\usepackage{bbding}

\theoremstyle{plain}
\newtheorem{theorem}{Theorem}
\newtheorem{lemma}{Lemma}
\newtheorem{corollary}{Corollary}
\newtheorem{assumption}{Assumption}
\newtheorem{definition}{Definition}

\DeclareMathOperator*{\argmin}{argmin} % no space, limits underneath in displays

\def\bx{\boldsymbol{x}}

\def\bz{\boldsymbol{z}}
\def\bb{\boldsymbol{b}}

\def\bs{\boldsymbol{s}}

\def\0{\mathbf{0}}
\def\1{\mathbf{1}}
\def\cA{\mathcal{A}}

\def\cE{\mathcal{E}}
\def\cF{\mathcal{F}}

\def\cH{\mathcal{H}}
\def\cL{\mathcal{L}}
\def\cN{\mathcal{N}}
\def\cP{\mathcal{P}}
\def\cR{\mathcal{R}}
\def\cX{\mathcal{X}}
\def\cY{\mathcal{Y}}
\def\cZ{\mathcal{Z}}

\def\Rbb{\mathbb{R}}

\def\Ebb{\mathbb{E}}

\def\sign{\operatorname{sign}}
\def\Lip{\operatorname{Lip}}

\DeclareMathOperator*{\esssup}{ess\,sup}

\def\tE{ \widetilde{\mathcal{E}}}
\def\bdelta{\bm{\delta}}
\begin{document}
\begin{frontmatter}
	\title{Non-Asymptotic Bounds for Adversarial Excess Risk under Misspecified Models}

		\runtitle{Adversarial Excess Risk under Misspecified Models}
	
		\begin{aug}
			%%%%%%%%%%%%%%%%%%%%%%%%%%%%%%%%%%%%%%%%%%%%%%
			%%Only one address is permitted per author. %%
			%%Only division, organization and e-mail is %%
			%%included in the address.                  %%
			%%Additional information can be included in %%
			%%the Acknowledgments section if necessary. %%
			%%%%%%%%%%%%%%%%%%%%%%%%%%%%%%%%%%%%%%%%%%%%%%
\author[A]{\fnms{Changyu} \snm{Liu}\thanks{Equal contributions}\ead[label=e1,mark]{changyuliu@cuhk.edu.hk}}
\author[B]{\fnms{Yuling} \snm{Jiao}$^{*}$\ead[label=e2,mark]{yulingjiaomath@whu.edu.cn}}
\author[C]{\fnms{Junhui} \snm{Wang}\ead[label=e3,mark]{junhuiwang@cuhk.edu.hk}}
			\and
\author[D]{\fnms{Jian} \snm{Huang}\ead[label=e4,mark]{j.huang@polyu.edu.hk}}
			%%%%%%%%%%%%%%%%%%%%%%%%%%%%%%%%%%%%%%%%%%%%%%
			%% Addresses                                %%
%%%%%%%%%%%%%%%%%%%%%%%%%%%%%%%%%%%%%%%%%%%%%%
			\address[A]{Department of Statistics, The Chinese University of Hong Kong,  Hong Kong SAR, China\\  \printead{e1}}			
			
	\address[B]{School of Mathematics and Statistics, Wuhan University, Wuhan, Hubei, China\\
\printead{e2}}

				\address[C]{Department of Statistics, The Chinese University of Hong Kong,  Hong Kong SAR, China\\ \printead{e3}}	
			
\address[D]{Department of Applied Mathematics, The Hong Kong Polytechnic University, Hong Kong SAR, China\\  \printead{e4}}
		\end{aug}
%\end{frontmatter}

\begin{abstract}
We propose a general approach to evaluating the performance of robust estimators based on  adversarial losses under misspecified models. We first show that adversarial risk is equivalent to the risk induced by a distributional adversarial attack under certain smoothness conditions.
This ensures that the adversarial training procedure is well-defined. To evaluate the generalization performance of the adversarial estimator, we study the adversarial excess risk. Our proposed analysis method includes investigations on both generalization error and approximation error.   We then establish non-asymptotic upper bounds for the adversarial excess risk associated with Lipschitz loss functions. In addition, we apply our general results to adversarial training for classification and regression problems. For the quadratic loss in nonparametric regression, we show that the adversarial excess risk bound can be improved over those for a general loss.
\end{abstract}


% REQUIRED
\end{frontmatter}
% REQUIRED
%\begin{MSCcodes}
%62G05, 62G35, 68T07
%\end{MSCcodes}

\section{Introduction}
Deep learning methods are known to be vulnerable to adversarial examples, which
%are formed by applying an
can result from
imperceptible perturbations to the input data.
 A deep learning model trained using such perturbed input data
 %may cause the model to
 may make highly confident but erroneous predictions \citep{Szegedy2014Intriguing,GoodfellowSS14}. This problem has gained widespread attention in recent years.
%Both
Methods for finding adversarial attacks \citep{GoodfellowSS14,papernot2016limitations,
Moosavi2016CVPR,carlini2017towards,brendel2018decisionbased,Athalye18a,Pixel2019}
and  developing adversarial defense \citep{papernot2016distillation,madry2018towards,zhang19p,Cohen19c}
have been extensively studied.
%proposed.
Among the adversarial defense methods, adversarial training has been empirically shown to be successful \citep{madry2018towards}.

%Despite the
Although there has been significant progress
% has been made
in developing methods for defending adversarial attacks,
% and defense,
%there is a lack of systematic analysis of the theoretical properties of
theoretical understanding of adversarial robustness remains limited. \cite{Pydi2021a,Pydi2021} considered the classification loss under the adversarial binary classification setting and obtained
%showed
the optimal adversarial classification risk.  \cite{Pydi2021,Awasthi2021}, and  \cite{bungert2021geometry} proved the existence and minimax
%theorem
properties for the adversarial classification risk. The results were later extended by \cite{frank2022existence} to the setting where surrogate functions were used.  Another series of work investigated the calibration and consistency of the surrogate loss functions under adversarial attacks
%setting
\citep{Bao2020,AwasthiConsistency2021,awasthi2021finer,meunier2022towards}.

Several authors have considered the generalization errors of  adversarial estimators in recent years.
%Analyzing the generalization error of  adversarial estimator has also received  significant attention  in %recent years.
%The works of
Examples include
\cite{Yin19b,khim2018adversarial,Awasthi20a}, and \cite{mustafa22a}, who analyzed the Rademacher complexity of adversarial loss function class.   \cite{Tu2019} transformed the adversarial learning into a distributional robustness optimization (DRO) problem and studied its generalization properties. However, the above work only considered well specified models, i.e., the underlying target function is assumed to belong to a class of  neural network functions.
%the under
%only consider the  , which is the gap between the  population adversarial  risk and the empirical adversarial risk. This is not sufficient to characterize the generalization performance.
% To understand the gap, it is helpful to consider an extreme case where the estimation function class only has one element which is far from the target function.  Due to the low complexity of the function class, the generalization error would be small, which can just follow from the central limit theorem. However, a large adversarial risk for the function would still be observed.   The error induced by  the   distance between the target
As well known in classical nonparametric method for classification and regression, the underlying  target is the  link function defined as $\mathbb{E}[Y|X=x]$  which is not a neural network function
in general.
%function and the space of estimation class is known as  the approximation error, which is also called the  error in some literature.
%Consequently
Therefore, a natural question is what are the properties of adversarial estimator under misspecified models, i.e., when the underlying target function is not an exact neural network function, but can only be approximated by neural networks. Under this more general setting, it is necessary to consider both the generalization error and the approximation error caused by model misspecification.
%whether we  can  provide theoretical guarantee by studying   both generalization error and approximation %error (the error caused by model misspecification) ?
In this paper, we study this problem systematically.
%give a comprehensive explore on this question.
We first provide a summary of the main features of our result and the related ones in Table
\ref{table_compare} below.

%are significant in the study of the generalization performance.

While  adversarial training improves the
%performance
robustness of an estimator on adversarially perturbed data, this benefit often comes at the cost of more resource consumption and leads to a reduction of accuracy on natural unperturbed data \citep{tsipras2018robustness}. Some recent works have tried to gain theoretical understanding of the trade-offs between accuracy and robustness \citep{madry2018towards,Schmidt2018data,raghunathan2018certified,
	tsipras2018robustness,Duchi2019adversarial,dobriban2020provable,nma2022on, zhang19p,Javanmard_linear, Javanmard_classification}. However, none of the above mentioned works studied  the setting of deep adversarial training with misspecified models.

\begin{table}[H]
	\caption{\label{table_compare}Comparison of recent methods for studying generalization performance of    adversarial estimator. The $\ell_r$ attack refers to general $\ell_r$ adversarial attack for $r\geq 1$,    the error $\cE_{gen}$ refers to the generalization error, and $ \cE_{app}$ refers to the approximation  error. }
	\centering
	\begin{tabular}{lcccc}
		\toprule
		&$\ell_r$ attack &  FNNs & $\cE_{gen}$ & $ \cE_{app}$\\
		\midrule
		\cite{Yin19b}&\xmark&\xmark&\cmark&\xmark\\	
		\midrule
		\cite{khim2018adversarial}&\cmark & \cmark  & \cmark&\xmark\\
		\midrule
		\cite{Awasthi20a}&\cmark & \xmark  & \cmark&\xmark\\
		\midrule
		\cite{mustafa22a}&\cmark&\cmark&\cmark&\xmark\\
		\midrule
		\cite{Tu2019}&\cmark&\cmark&\cmark&\xmark\\	
		\midrule
		This work &\cmark&\cmark&\cmark&\cmark\\
		\bottomrule
	\end{tabular}
\end{table}

In this work,  we provide theoretical guarantees for
%the performance of
deep adversarial training with misspecified models
%through
%comprehensively
by establishing non-asymptotic error bounds for
%analyzing
the  adversarial excess risk,  defined as the difference between the adversarial risk of adversarial estimator and the optimal adversarial risk.
%We drive non-asymptotic error bounds on the adversarial excess risk by trading off the triple  (bias, variance, robustness) i.e., the approximation error, the generalization error and the adversarial attack level.
%we propose a general analysis method to study the generalization performance of   adversarial estimator.
We show that the adversarial excess risk
%the generalization performance of adversarial estimator can be comprehensively analyzed  based on the study of the adversarial excess risk.  To investigate the  adversarial excess risk,
can be decomposed as
\begin{equation*}
\text{Adversarial excess risk}\leq \cE_{gen}+\cE_{app},
\end{equation*}
where $\cE_{gen}$ represents the generalization error and $\cE_{app}$ represents the  approximation error. The approximation error $\cE_{app}$ is due to model misspecification when the target function is not in the class of neural networks used in training. As summarized in Table \ref{table_compare},
the existing related work has only considered the generalization error, assuming a well-specified model.
By investigating these two types of errors, we establish a non-asymptotic error bound on the adversarial excess risk in the more general setting that allows model misspecification.
%We also study the accuracy of adversarial estimator  by computing the expected loss of the estimator on  naturally unperturbed test data.
Our main contributions are summarized as follows.
\begin{itemize}
	%	\item We rigorously demonstrate that adversarial risk is well-defined in our studied setting and show that the adversarial risk is equivalent to a risk induced by a distributional adversarial attack under certain smoothness conditions.
	%\item A general analysis method of  the adversarial excess risk is proposed.
	\item  We establish non-asymptotic error bounds for the adversarial excess risks under misspecified models and use the feedforward neural networks (FNNs) with constraints on the Lipschitz property.  The  error bounds explicitly illustrate  the influence of the adversarial attack level and can achieve the rate $O(n^{-\alpha/(2d+3\alpha)})$ up to a logarithmic factor,  where
	$\alpha$ represents the smoothness level of the underlying target function and $d$
	is the dimension of input.   The structure of the neural network is specified to show when the error rate can be achieved.
	\item  We also evaluate the adversarial estimator under natural risk and local worst-case risk.  The result theoretically demonstrates that the adversarial robustness can hurt the  accuracy in a general nonparametric  setting.
	\item
	%The applicability of the results is verified by the application
	We apply our general results to the classification and nonparametric regression problems in an adversarial setting
	%, where we
	and establish non-asymptotic error bounds for the adversarial estimators under the adversarial classification risk and $L_2$-norm, respectively.
\end{itemize}

The results for error bounds for the adversarial estimator in different settings and measurements are summarized in   Table \ref{table_sumerror}.

The rest of the paper is organized as follows.  Section \ref{sec_adrisk} introduces the notations and  problem setup. Section \ref{sec_main} contains  main results of the paper. Section \ref{sec_app} presents applications of the  results  to classification and  regression problems.  In Section \ref{sec_related}, discussion on some related works is given.
Concluding remarks are given in Section \ref{sec_con}.
% we give a brief conclusion of the paper.
 The proof of main theorem is given in Appendix, and the remaining proofs are relegated to the supplementary material.

\begin{table}[H]
	\caption{\label{table_sumerror}Summary of error bounds (up to a logarithmic factor) in the paper, where $\varepsilon$ represents the adversarial attack level, $r_1=\alpha/(2d+3\alpha),$ $r_2=2\alpha/(2d+5\alpha),$ $r_3=(d+3\alpha-1)/(2d+3\alpha)$, $r_4=(d+1)/(2d+3\alpha),$  and $r_5=(d+1)/(2d+5\alpha).$}
	\centering
	\begin{tabular}{lll}
		\toprule
		Loss function   & Measurement  & Error bound \\
		\midrule
		\multirow{3}{*}{
			Lipschitz
		}    &adversarial excess risk&$   n^{-r_1} +n^{-r_3}\varepsilon$ \\	
		\cmidrule{2-3}
		& excess risk&$n^{-r_1} +\varepsilon$\\
		\cmidrule{2-3}
		&local worst-case excess risk&$n^{-r_1}  +  n^{r_4} \varepsilon$\\
		\midrule Classification
		&    adversarial excess risk& $n^{-r_1}  + \varepsilon$\\
		\midrule
		Quadratic & $L_2$-norm &$n^{-r_2} +n^{r_5}\varepsilon$\\
		\bottomrule
	\end{tabular}
\end{table}

\subsection*{Notation}
Let the set of positive integers be denoted by $\mathbb{N}= \{1, 2, . . .\}$ and let  $\mathbb{N}_0=\mathbb{N}\cup \{0\}.$  If $a$ and $b$ are two quantities,
we use $a\lesssim b $ or $b\gtrsim a$ to denote the statement that $a\leq Cb$ for some constant $C>0$. We denote $a\asymp b$ when $a\lesssim b\lesssim a$.
Let $\lceil a \rceil$ denote the smallest integer larger than or equal to  quantity $a$. For a vector $\bx$ and $p\in[1,\infty]$, we use $\|\bx \|_{p}$ to denote the $p$-norm of $\bx$. For a function $f$, we use $\|f \|_{\infty}$ to denote the supremum norm of $f$.

\section{Problem setup}\label{sec_adrisk}
In this section, we introduce the definition of adversarial risk and present a basic setup of adversarial training. We also
%establish a theoretical
lay the foundation for the % subsequent
theoretical analysis of adversarial training and establish some
%where
basic properties of adversarial risk.
% are established.
%shown.

\subsection{Adversarial risk}
Suppose that $(X,Y)$ follows an unknown distribution $P$ over $\cZ= \cX\times\cY$, where $\cX\subseteq\Rbb^d$ and $\cY\subseteq\Rbb.$
For a loss function $\ell:\Rbb\times \cY\mapsto [0,\infty)$ and a measurable function $f: \cX\mapsto \Rbb$,   the (population) natural  risk is defined by
\begin{equation*}
\cR_{P}(f)=	\Ebb_{( X,Y)\sim P}\big[  \ell (f( X),Y )\big].
\end{equation*}
To evaluate the performance of function $f$ in the presence of  adversarial attacks,  the (population) adversarial risk is defined by
\begin{align*}
\widetilde\cR_{P}(f,\varepsilon)=	\Ebb_{(X,Y)\sim P} \big[\sup_{X^{\prime} \in B_{\varepsilon}(X) }\ell (f(X^{\prime}),Y ) \big],
\end{align*}
where  $B_{\varepsilon}(\bx)=\left\{\bx^{\prime}\in\cX:\|\bx^{\prime}-\bx\|_{\infty}\leq \varepsilon\right\}.$   Here we  focus on $\ell_{\infty}$ attack.
%Further
 In Section \ref{sec_main}, we will show that the proposed analysis method can be easily extended to a general
$\ell_{r}$ attack.

\subsection{Properties of adversarial risk}
Adversarial risk has been widely considered in recent years for the goal of deriving adversarial robust estimator.  To facilitate the analysis, we make the following assumptions.
\begin{assumption}\label{ass1}
	$\cZ\subseteq [0,1]^{d}\times[-1,1]$ and $\cup_{\bx\in\cX}B_{\varepsilon}(\bx)\subseteq [0,1]^d$ holds for $\varepsilon>0$.
\end{assumption}
The assumption $\cup_{\bx\in\cX}B_{\varepsilon}(\bx)\subseteq [0,1]^d$ is to  guarantee  that the estimation  function class  is well-defined under the adversarial setting.  Our analysis can be easily extended to a more general setting, where $\cZ$ is bounded.
For a loss function $\ell:\Rbb\times \cY\mapsto [0,\infty)$, we define
\begin{equation*}
\Lip^1(\ell)=\sup_{y\in\cY}\sup_{u_1\ne u_2}\frac{|\ell(u_1,y)-\ell(u_2,y)|}{|u_1-u_2|}.
\end{equation*}
\begin{assumption}\label{ass2}
	The loss function is continuous and satisfies $\Lip^1(\ell)<\infty.$
\end{assumption}
The assumption $\Lip^1(\ell)<\infty$ is weaker than the Lipschitz continuity condition,  since it only imposes restriction on the Lipschitz constant of $\ell(\cdot, y)$ for every  $y\in\cY$. The assumption
%can be
is satisfied by many commonly used loss functions, such as the hinge loss and $\rho$-margin loss.

We first
%demonstrate
show that the adversarial risk is well-defined in our
%studied
setting.  {\color{black}This is
%significant
necessary since the adversarial risk
	%is not
	may not be well-defined in general \cite{Pydi2021}.} Specifically, we show
that the adversarial risk for a function $f$ can be represented by a natural risk with the expectation taken over a shifted distribution. Moreover,  the distance between the shifted distribution and the data generating distribution can be measured by the $\infty$-th Wasserstein distance.
%The definition of the Wasserstein distance is given as below.

Let $d_{\cZ}$ denote a  metric over $\cZ$ satisfying $d_{\cZ}(\bz_1,\bz_2)=\|\bx_1-\bx_2\|_{\infty}+|y_1-y_2|$ for any $\bz_1=(\bx_1,y_1)$ and $\bz_2=(\bx_2,y_2)\in\cZ.$
Let $ \cP(\cZ)$ denote the space of Borel probability measures on $\cZ$. For $p\in[1,\infty)$,  the $p$-th Wasserstein distance  between two probability measures $P,Q\in\cP(\cZ)$ is defined as
\begin{align*}
W_p(P,Q)&=\Big\{\inf_{\pi\in\Pi(P,Q)}\Ebb_{(Z_1,Z_2)\sim \pi}\big[d_{\cZ}(Z_1,Z_2)^p\big]\Big\}^{1/p},
\end{align*}
where $\Pi(P,Q)$ denotes the collection of all probability measures on $\cZ\times\cZ$ with marginals $P$ and $Q.$  The
$\infty$-th Wasserstein distance is defined to be the limit of
the $p$-th Wasserstein distances, which can also  be characterized by
\begin{align*}
W_\infty(P,Q)
= \inf_{\pi\in\Pi(P,Q)} \esssup_{(Z_1,Z_2)\sim\pi} d_{\cZ}(Z_1,Z_2).
\end{align*}
Since $W_p(P,Q) \leq  W_q(P,Q) $ for any $1\leq p\leq q\leq \infty,$ the $\infty$-th Wasserstein distance is stronger than any $p$-th Wasserstein distance. Similarly, for $p\in[1,\infty]$, we define the $p$-th Wasserstein distance over $ \cP(\cX)$ based on the supremum norm, where $ \cP(\cX)$ denotes the space of Borel probability measures on $\cX$.

\begin{lemma}\label{lem:nat}
	Suppose Assumption \ref{ass1} holds, $\ell$ and $f$ are continuous,   there exists a measurable function   $T^{\star}$  satisfying $T^{\star}(\bz)\in B_{\varepsilon}(\bx) $ such that $
	\sup_{\bx^{\prime}\in B_{\varepsilon}(\bx)} \ell(f(\bx^{\prime}), y)=\ell(f(T^{\star}(\bz)), y).
	$
	Let the joint distribution of $(T^{\star}(Z), Y)$  be denoted by $P^{\star}$, we have $W_{\infty}(P^{\star},P)\leq \varepsilon$ and
	\[\widetilde \cR_{P}(f,\varepsilon)=\cR_{P^{\star}}(f).\]
\end{lemma}

%Based on
Lemma \ref{lem:nat} shows  that the adversarial risk is well-defined with respect to the shifted distribution $P^{\star}$.
With further analysis of this shifted distribution, we construct an equivalent relationship between the adversarial risk and the risk induced by the distribution perturbing adversary \citep{Pydi2021,Pydi2021a}. The result is given  in Section \ref{ap_relation} below.

To state the next lemma, we denote
the Lipschitz constant for a function $f$ by
%is defined by
\[
\Lip(f)=\sup_{\bx_1\ne \bx_2}\frac{|f(\bx_1)-f(\bx_2)|}{\|\bx_1-\bx_2\|_{\infty}}.
\]
\begin{lemma}\label{lem:Req}
	Suppose Assumptions \ref{ass1}--\ref{ass2} hold  and  $\Lip(f)< \infty$, then
	\begin{align*}
	\cR_{P}(f)\leq \widetilde \cR_{P}(f,\varepsilon)\leq \cR_{P}(f)+\Lip^1(\ell)\Lip(f)\varepsilon.
	\end{align*}
\end{lemma}
%The result of
Lemma \ref{lem:Req} %suggests
shows that the adversarial robustness
%can be
is related %with
to the Lipschitz constraint.
%{\color{red}Under what condition the upper bound and lower bound can be reached?}
  %This has been confirmed in
	\cite{cisse2017parseval,tsuzuku2018lipschitz} and \cite{bubeck2021a}
	also showed that robustness is related to the Lipschitz constraint under different settings.

%  It inspires our design of the adversarial training setup and  will be helpful in subsequent analysis.

\subsection{Relationship between adversarial risk and distribution perturbing risk}
\label{ap_relation}
%	For a function $f:[0,1]^d\mapsto\Rbb$, the Lipschitz constant is defined by
%	\[
%	\Lip(f)=\sup_{\bx_1\ne \bx_2}\frac{|f(\bx_1)-f(\bx_2)|}{\|\bx_1-\bx_2\|_{\infty}}.
%	\]
%	Leveraging the Lipschitz property of functions $\ell$ and $f$, we show the following result.
%With further study
We  further study the shifted distribution in Lemma \ref{lem:nat}
%, we
and  construct a relationship between the adversarial risk and another kind of risk induced by a distributional adversarial attack defined below.
%This new kind of  risk is introduced  below.

For any distribution $Q\in \cP(\cZ)$, we denote its corresponding pair of variables by $(\widetilde{X},\widetilde{Y})$, and let the conditional distribution of $\widetilde{X}$ given $\widetilde{Y}=y$ be denoted by $Q_y $ for every $y\in\cY$. The collection of distributions $\Gamma_{\varepsilon}$ is defined by
\begin{align*}
\Gamma_{\varepsilon} =\Big\{Q\in \cP(\cZ): &\text{ when } (\widetilde{X},\widetilde{Y})\sim Q,\text{ then }
\widetilde{Y}\sim P_{Y}  \text{ and } W_\infty(Q_y,P_y)\leq \varepsilon, \forall y\in\cY
\Big\},
\end{align*}
where $P_y$ is the conditional distribution of $X$   given $Y=y$ and $P_Y$ is the distribution of $Y.$   Intuitively, it is helpful to think $\Gamma_{\varepsilon}$ as a collection of distributional adversarial attacks. For every $Q\in \Gamma_{\varepsilon}$, as it observing sample $(X,Y)$,  with $y$ denoting the  value of $Y$,  it perturbs $X$ to $\widetilde{X}$ such that $\widetilde{X}\sim Q_y$ and lets $ Q_y$ lie  in an uncertainty set around $P_y.$  This kind of adversarial attack strategy is also known as distribution perturbing adversary  \citep{Pydi2021,Pydi2021a}. The corresponding distribution perturbing risk is defined as $ \sup_{Q\in\Gamma_{\varepsilon}}\cR_{Q}(f).$

\begin{theorem}\label{lem:disat}
	Suppose Assumption \ref{ass1} holds, $\ell$ and $f$ are continuous, then  we have $
	\widetilde \cR_{P}(f,\varepsilon)=  \sup_{Q\in\Gamma_{\varepsilon}}\cR_{Q}(f).
	$
\end{theorem}
%Therefore, we  demonstrate

This theorem shows that the risks induced by the  two different types of adversarial attack  are equivalent under the smoothness condition. Similar results were constructed in \cite{Pydi2021,Pydi2021a}, while \cite{Pydi2021} focused on the binary classification setting and  \cite{Pydi2021a} considered the case when $\cY$ is a discrete set of labels. % {\color{red} Gar rui??}
 From the proof of Theorem \ref{lem:disat}, we
also show that $P^{\star}\in\Gamma_{\varepsilon}$, which directly implies  $W_{\infty}(P^{\star},P)\leq \varepsilon.$ Hence,  it
%illustrates
shows a stronger relationship with $P$.

\subsection{Adversarial training}
% { %To derive an estimator that is robust to adversarial attacks,
Adversarial  training
%%is commonly applied. It
aims to
%%find
learn a target function $f^{\star}$ that  minimizes the adversarial  risk.
%% over an interested space $\cH$.  In the work,  we focus on the
In this work, we assume that $f^{\star}$ belongs to a H\"{o}lder class.

\begin{definition}[H\"{o}lder  class]
	Let $d\in\mathbb{N}$ and $\alpha=r+\beta>0$, where $r\in\mathbb{N}_0$ and $\beta\in(0,1]$. Let  $\bs \in \mathbb{N}_{0}^{d}$ denote the multi-index. The H\"{o}lder  class $\cH^{\alpha}(\Rbb^d)$ is defined as
	\begin{equation*}
	\begin{split}
	\mathcal{H}^{\alpha} (\Rbb^{d} )=\Big\{f: \Rbb^{d} \rightarrow \Rbb,
&
\max_{\|\bs\|_{1} \leq r} \sup _{\bx \in \mathbb{R}^{d}}\left|\partial^{\bs} f(\bx)\right| \leq 1,
\\&
	\max _{\|\bs\|_{1}=r} \sup _{\bx_1 \neq \bx_2} \frac{\left|\partial^{\bs} f(\bx_1)-\partial^{\bs} f(\bx_2)\right|}{\|\bx_1-\bx_2\|_{\infty}^{\beta}} \leq 1\Big\}.
	\end{split}
	\end{equation*}
We 	let  $\mathcal{H}^{\alpha}=\left\{f:[0,1]^{d} \rightarrow \mathbb{R}, f \in \mathcal{H}^{\alpha} (\mathbb{R}^{d} )\right\}$ denote the restriction of $\mathcal{H}^{\alpha} (\mathbb{R}^{d} )$ to $[0,1]^{d}$.
\end{definition}
Our target function $f^{\star}$ is defined by
\begin{equation}\label{eq:target}
f^{\star}\in\argmin_{f\in \mathcal{H}^{\alpha}} \widetilde\cR_{P}(f,\varepsilon).
\end{equation}
When only a finite sample $\{(X_i, Y_i), i=1, \ldots, n\}$  is available,  we estimate $f^{\star}$ by  minimizing the empirical adversarial  risk   over a space of estimation functions $\cF_n$, which can vary with  $n$. Specifically, we aim to find an estimator $\widehat{f}_n \in \cF_n$ that solves
\begin{equation*}
\widehat{f}_n\in\argmin_{f\in \cF_n}\widetilde  \cR_{P_n}(f,\varepsilon), \quad \text{where }\widetilde\cR_{P_n}(f,\varepsilon)=\frac{1}{n}\sum_{i=1}^{n}\big[ \sup_{ X^{\prime}_i \in B_{\varepsilon}(X_i) }\ell(f (X^{\prime}_i) ,Y_i)\big].
\end{equation*}
Here $P_n$ denotes the empirical distribution induced by the samples.
The function $\widehat{f}_n$ is called an adversarial estimator.
%We focus on the feedforward neural networks with constraints on the Lipschitz property, based on the relationship between Lipschitz constraints and adversarial robustness.
  Based on the relationship between Lipschitz constraints  and  adversarial robustness, we focus on  the feedforward neural network with  constraints on Lipschitz property.
%The specific definition is introduced as follows.

\subsection{Feedforward neural networks with norm constraints}
A  feedforward neural networks (FNN) can be represented in the form of
\[
g=g_{L}\circ g_{L-1}\circ\cdots\circ g_{0},
\]
where $g_{i}(\bx)=\sigma(A_i\bx+\bb_i)$  and  $g_{L}(\bx)=A_{L}\bx$, with  $A_{i}\in\Rbb^{d_{i+1}\times d_{i}}$ and $\bb_{i}\in\Rbb^{d_{i+1}\times 1}$  for $i=0,\dots, L-1$, $A_{L}\in\Rbb^{d_{L+1}\times d_{L}}$, and $\sigma(x)=\max\{x,0\}$ being the ReLU activation function (applied
component-wise).  For simplicity of notation, we use $g_{\theta}$ to emphasis that the  FNN  is parameterized by $\theta=(A_0,\dots, A_{L}, \bb_{0},\dots, \bb_{L-1})$.
The number $W=\max\{d_1,\dots, d_L\}$ and $L$ are called the width and depth of the FNN, respectively. We let  $\cN\cN(W,L)$ denote the class of FNNs with width $W$ and depth $L$. Additionally, we
define  $\cN\cN(W,L,K)$    as the subset of functions in $\cN\cN(W,L)$ which satisfies the following norm constraint on the weight.
\begin{equation*}
\kappa(\theta):=\|A_L\|\prod_{i=0}^{L-1}\max\{\|(A_{i},\bb_i)\|,1\}\leq K,
\end{equation*}
where the  norm satisfies  $\|A\|=\sup_{\|\bx\|_{\infty}\leq 1}\|A\bx\|_{\infty}$.
Then for any $g_{\theta}\in\cN\cN(W,L,K)$, we have
\begin{equation*}
\Lip(g_{\theta})\leq \kappa(\theta)\leq K.
\end{equation*}
Therefore,  the Lipschitz constants of the functions in $\cN\cN(W,L,K)$ have an uniformly upper bound.
%Computationally,

\section{Non-asymptotic error bounds}%Main results}
\label{sec_main}
In this section, we  present our main results of  non-asymptotic
bounds for the adversarial excess risk.
%We briefly describe the idea of the proposed analysis method and show its applicability in general setting.
%We discuss the connections and differences between our results and some related results in detail.
We also discuss the relationship between accuracy and adversarial robustness in the sense
that
more robustness can lead to less accurate upper bounds for the excess risk.

\subsection{Non-asymptotic error bounds for adversarial excess risk}
%We work on the adversarial training  using the FNNs with norm constraints.
The adversarial estimator based on FNNs with norm constraints is defined by
\begin{equation}\label{eq:esti}
\widehat{f}_n\in\argmin_{f\in \cN\cN(W,L,K)}\widetilde  \cR_{P_n}(f,\varepsilon).
\end{equation}
We evaluate  its performance  via the adversarial excess risk
\[
\tE(\widehat{f}_n,\varepsilon) =\widetilde\cR_{P}(\widehat{f}_n,\varepsilon) -\inf_{f\in \mathcal{H}^{\alpha}\cup \cN\cN(W,L,K)} \widetilde\cR_{P}(f,\varepsilon).
\]
The adversarial excess risk is  nonnegative.   It
%focuses on
is a measure for evaluating the performance of  an adversarial estimator for future data
%by studying the population  adversarial  risk and  takes
using the optimal population adversarial risk as a benchmark.
% to provide  an accurate evaluation.
% Therefore,  it  enables an precise and comprehensive assessment of the generalization performance of %adversarial estimator.

To investigate the adversarial excess risk, we show  it can be decomposed  into  \eqref{eq_qq},
where $\cE_{gen}$ represents the generalization error, which is the difference between the population adversarial risk and empirical adversarial risk, and $\cE_{app}$ represents the approximation  error due to model misspecification, which measures the distance between the target function and the space of estimation functions that may not contain the target function.  By investigating both  $\cE_{gen}$  and  $\cE_{app}$, we establish a  non-asymptotic error bound on the adversarial excess risk.
%, and we specify the structure of the neural network when the error rate can be achieved. The formal result is given below.

\begin{theorem}\label{thm:main1}
Consider a H\"{o}lder space $\cH^{\alpha}$  with  $\alpha=r+\beta\geq 1$, where  $r\in\mathbb{N}_0$ and $\beta\in(0,1]$. Let $\gamma=\lceil \log_2 (d+r)\rceil$. Suppose Assumptions \ref{ass1}--\ref{ass2} hold. Suppose
 $W\geq c(K/\log^{\gamma}K)^{(2d+\alpha)/(2d+2)}$ for a constant $c>0$, and $L\geq 4\gamma+2$. Then for any adversarial estimator $\widehat{f}_n$ in \eqref{eq:esti} and adversarial attack level $\varepsilon>0$, we have
		\begin{equation}\label{eq_qq}
		\tE(\widehat{f}_n,\varepsilon) \lesssim \cE_{gen} + \cE_{app},
		\end{equation}
		where
		\begin{align*}
		\cE_{gen}&=  K\varepsilon n^{-1}+ WL\sqrt{\log (W^2L)} n^{-1/2} \sqrt{\log n}
		+n^{-\min\{1/2, \alpha/d  \}}\log^{c(\alpha,d)}n,\\
		\cE_{app}& = ( K/\log^{\gamma}K  )^{- \alpha/(d+1)}.
		\end{align*}
	\iffalse
	{\color{blue}
		\begin{equation}
		\begin{split}
		\tE(\widehat{f}_n,\varepsilon)&\lesssim  K\varepsilon n^{-1}+ WL\sqrt{\log (W^2L)} n^{-1/2} \sqrt{\log  n }
		\\&\quad +( K/\log^{\gamma}K  )^{- \alpha/(d+1)}
		\\&\quad+n^{-\min\{1/2, \alpha/d  \}}\log^{c(\alpha,d)}n.
		\end{split}
		\end{equation}
		[[[ specify $ \cE_{gen}$ and $\cE_{app}$].]]
	}
	\fi
	Here $c(\alpha,d)=1$ when $ d=2\alpha$, and $c(\alpha,d)=0,$ otherwise. 	

If  we further select
	$ K\asymp n^{(d+1)/(2d+3\alpha)}$ and $WL\asymp n^{(2d+\alpha)/(4d+6\alpha)},$ then we   have
	\begin{equation}\label{eq_errorbound}
	\tE(\widehat{f}_n,\varepsilon)  \lesssim  n^{-(d+3\alpha-1)/(2d+3\alpha)}\varepsilon+n^{-\alpha/(2d+3\alpha)}\log n^{\xi}.
	\end{equation}
	where $\xi=\max\{1,\gamma\alpha/(d+1)\}$.
\end{theorem}

	The upper bound \eqref{eq_qq}
	%on the  adversarial excess risk  is derived from a sum of non-asymptotic upper bounds on the errors
	%$ \cE_{gen}$ and $\cE_{app}$, where the  upper bounds are
	is determined by the smoothness property of the H\"{o}lder space,  the  structure of the estimation function class $\cN\cN(W,L,K)$,  the adversarial  attack level $\varepsilon$, and sample size $n$.
%	The errors $I_1,I_3$ and $I_4$ contribute to   the generalization error $ \cE_{gen}$, while the error $I_2$ contributes to  the approximation error $\cE_{app}$.
	There is a trade-off between the two errors. Specifically, $ \cE_{gen}$  increases with the complexity of $\cN\cN(W,L,K)$,  with larger $W,L$ and $K$ leading to a larger upper bound. On the other hand, as
	%the function class
	$K$ becomes larger,  the error $\cE_{app}$ decreases.  To achieve the best error rate, we balance the trade-off between the two errors and show the  rate can reach $n^{-\alpha/(2d+3\alpha)}$  up to a logarithmic factor for suitable chosen $\varepsilon$.

The proposed analysis
%technique
method
%can be
is applicable to
%different
the  settings with
% various
different models, loss functions, estimation function classes, and adversarial attacks.  For example, we can apply the method to the setting where a general $\ell_r$ adversarial attack is used.  The  upper bounds on  the corresponding generalization error and approximation error can
%follow
be obtained based on  the results on Rademacher complexity of general adversarial loss function classes \citep{Awasthi20a,mustafa22a} and  the results on approximation power of different estimation function classes  such as the deep neural networks \citep{Yarotsky2018,lu2021deep,jiao2022approximation}.

The result \eqref{eq_errorbound}
%illustrates that
show how the bounds for the
% error rate of
 adversarial excess risk
%is affected by
depends on the adversarial attack level $\varepsilon$,
where $\varepsilon$ is allowed to vary with $n$.  Let $e_n=n^{(d+2\alpha-1)/(2d+3\alpha)}.$ When $\varepsilon =O(e_n)$, the error rate  can reach $n^{-\alpha/(2d+3\alpha)}$  up to a logarithmic factor. However, if $\varepsilon$ grows faster than $e_n$, the error rate of $\tE(\widehat{f}_n,\varepsilon)$ is dominated by the rate of $\varepsilon.$ Moreover, the convergence of $\tE(\widehat{f}_n,\varepsilon)$ cannot be guaranteed  when $\varepsilon$ grows faster than $n^{ (d+3\alpha-1)/(2d+3\alpha)}.$

As mentioned above,  the error rate of the adversarial excess risk can reach $n^{-\alpha/(2d+3\alpha)}$  up to a logarithmic factor when $\varepsilon$ is appropriately selected. Here we only require  the Lipschitz property of the loss function.  We will further show in Section \ref{sec_reg} that the error rate can be improved to $n^{-2\alpha/(2d+5\alpha)}$  up to a logarithmic factor when using the
% square
quadratic loss, where the
% enhancement
improvement is due to an improved approximation error bound.

Some recent papers have studied the convergence properties of deep neural network under the excess risk  \eqref{eq_excessrisk},  where the data are naturally unperturbed \citep{schmidt2020nonparametric,bauer2019deep,jiao2021deep}. The results are generally established under certain smoothness assumption on the target function. And it is typically assumed that the target function is in a H\"{o}lder class with a smoothness index $\alpha.$ The results show that the deep neural network estimation could achieve the optimal minimax rate $n^{-2\alpha/(d+2\alpha)}$ established by \cite{stone1982optimal}. Though the structures of neural networks vary in these works, which include  different choices of width, depth, and activation functions,  they make no constraint on the Lipschitz property of neural networks.   \cite{jiao2022approximation} investigated the approximation properties of FNNs with norm constraints.    Intuitively, a norm constrained neural network class would be smaller in size compared to an unconstrained neural network class with the same structure. Therefore, the benefit of the Lipschitz property comes at the cost of losing the approximation power, which would lead to larger approximation errors.  This is demonstrated in \cite{jiao2022approximation}, where the error rate of the excess risk only reaches the rate  $n^{-\alpha/(d+2\alpha+1)}$ up to a logarithmic factor.
%All the above results are built in the natural training.
%Focusing on adversarial training, adversarial attacks make analysis more difficult to perform. As a result, %our work is challenged by a harder control of error rate.
The above discussion is summarized in Table \ref{table_ratecom}.

%on the connections and differences between our work and some related works in terms of the non-asymptotic error bounds, the estimation function class, and the training setup.
%
%
\begin{table}[H]
	\centering
	\caption{\label{table_ratecom}Comparison of 
%non-asymptotic 
error bounds between our result and some related results (up to a logarithmic factor).}
	\begin{tabular}{llll}
		\toprule
		&	Risk & Estimation function class&  Error bound\\
		\midrule
	\tiny{	\cite{schmidt2020nonparametric,bauer2019deep,jiao2021deep}}&	natural & FNNs& $n^{-2\alpha/(d+2\alpha)}$\\
		\cite{jiao2022approximation}&	natural &  FNNs with norm constraints&$n^{-\alpha/(d+2\alpha+1)}$\\
		This paper &		adversarial & FNNs with norm constraints& $n^{-2\alpha/(2d+5\alpha)}$\\
		\bottomrule
	\end{tabular}
\end{table}

\subsection{Adversarial robustness can hurt accuracy}
%		Trade-off between accuracy and adversarial robustnes
We also evaluate the performance of the adversarial estimator under some other risks.
We first consider the natural risk
% is  first considered. It
%which characterizes the accuracy of an estimator
%by computing
%using the expected loss of the estimator on unperturbed test data.
%Specifically, we
and  study the excess risk defined by
\begin{equation}\label{eq_excessrisk}
\cE(\widehat{f}_n) = \cR_{P}(\widehat{f}_n) -\inf_{f\in \mathcal{H}^{\alpha}} \cR_{P}(f).
\end{equation}
%which is the difference between the accuracy  of the adversarial estimator and the optimal accuracy.
%We establish a non-asymptotic error bound on the excess risk.
% below.

\begin{corollary}\label{cor1}
	Suppose the
	%assumptions and
	conditions of Theorem \ref{thm:main1} are satisfied and $\alpha\geq 1$,
	%.  If we select
	%	$ K\asymp n^{(d+1)/(2d+3\alpha)}$ and $WL\asymp n^{(2d+\alpha)/(4d+6\alpha)}$,
	then   for  any adversarial estimator $\widehat{f}_n$  in  \eqref{eq:esti}, we have
	\[
	\cE(\widehat{f}_n)\lesssim n^{-\alpha/(2d+3\alpha)}\log n^{\xi}+\varepsilon,
	\]
	where $\xi=\max\{1,\gamma\alpha/(d+1)\}$.
\end{corollary}

Corollary \ref{cor1} shows that
the upper bound for the excess risk of the adversarial estimator is not guaranteed to converge.
The increase in the upper bound for the excess risk becomes significant when the adversarial robustness reaches a certain level. Previous studies have mostly focused on specific scenarios and made certain assumptions on the data distribution. For instance, \cite{Javanmard_classification, Javanmard_linear} analyzed the trade-offs in linear regression and binary classification with linear classifier, assuming the data was normally distributed. However, a comprehensive analysis of this problem is still lacking. We provide an upper bound for the excess risk that increases with the adversarial robustness level. Our result sheds light on the theoretical understanding of the trade-offs, but a complete analysis also requires lower bounds for the excess risk.

We
%also study
now consider the local worst-case risk.
%, which is extensively studied in the DRO.
Specifically, the local worst-case risk with $1$-th Wasserstein distance is defined by
\[
\cR_{P,1}(f,\varepsilon)=\sup_{Q:W_1(Q,P)\leq\varepsilon}\cR_{Q}(f),
\]
where the distribution $Q$ runs over an uncertainty set around the data generating distribution $P$. The excess risk with respect to the local worst-case risk is defined by
\[
\cE_{1}(\widehat{f}_n,\varepsilon) = \cR_{P,1}(\widehat{f}_n,\varepsilon)-\inf_{f\in \mathcal{H}^{\alpha}\cup \cN\cN(W,L,K)}\cR_{P,1}(f,\varepsilon).
\]
\begin{corollary}\label{cor2}
	Suppose the  conditions of Theorem \ref{thm:main1} are satisfied and $\Lip(\ell)<\infty,$
	%	 If  we select
	%	$ K\asymp n^{(d+1)/(2d+3\alpha)}$ and $WL\asymp n^{(2d+\alpha)/(4d+6\alpha)}$,
	then   for any adversarial estimator $\widehat{f}_n$ in \eqref{eq:esti}, we have
	\[
	\cE_{1}(\widehat{f}_n,\varepsilon)\lesssim n^{-\alpha/(2d+3\alpha)}\log n^{\xi}+K \varepsilon,
	\]
	where $\xi=\max\{1,\gamma\alpha/(d+1)\}$.
\end{corollary}

\section{Examples} %Applications}
\label{sec_app}
In this section, we consider the more specific settings of classification and regression and apply
%  the result of
Theorem \ref{thm:main1} to classification and regression problems.

\subsection{Classification}
Suppose that $(X,Y)$ follows an unknown distribution $P$ on $ \cX\times \{-1,1\}$.
A basic goal of binary classification is to predict the label $Y,$  when we only observe a predictor $X$ in a
%sample
%paired
random pair $(X,Y) \sim P$
%that is generated from $P$.
%To describe this learning target,
%one of the most
A commonly used loss
function
is the classification loss
%function
$\ell_{\operatorname{class}}  :\Rbb\times \{-1,1\}\mapsto [0,\infty),$
defined by $
\ell_{\operatorname{class}}(u,y)=\1\big\{\sign( u)y\leq 0\big\},$
where   $\sign(u) =1$ when $u\geq 0$, and $\sign(u) =-1$, otherwise.
%    Denote by
Let $f:\cX\mapsto \Rbb$  be a score function, and let the associated binary classifier be $\sign f(\cdot)$.   % function,
The natural classification risk and the adversarial classification risk of the score function $f$ are
%defined by
\begin{align*}
\cR_{\operatorname{class},P}(f)&=
%	\Ebb_{(\bX,Y)\sim P}  \ell_{\operatorname{class}} (f( \bX),Y )
\Ebb_{(X,Y)\sim P} \1\big\{\sign f(X)\ne Y\big\},
\\\widetilde\cR_{\operatorname{class},P}(f,\varepsilon)&=
%	\Ebb_{(\bX,Y)\sim P} \Big[\sup_{ \bX^{\prime} \in B_{\varepsilon}(\bX) %}\ell_{\operatorname{class}} (f( \bX^{\prime}),Y ) \Big]
\Ebb_{(X,Y)\sim P} \big[\sup_{ X^{\prime} \in B_{\varepsilon}(X) } \1\left\{\sign f(X^{\prime} )\ne Y\right\} \big].
% 	\\&  =\Ebb_{(X,Y)\sim P}  \1\big\{\exists X^{\prime} \in B_{\varepsilon}(X)  \text{ s.t. }  \sign f(X^{\prime})\ne Y  \big\}.
\end{align*}
Let $\eta(\bx)=P(Y=1| X=\bx)$.  Define $c_{\varepsilon}(\bx,\bx^{\prime})=\1\{\|\bx-\bx^{\prime}\|_{\infty}>2\varepsilon\}$    and let the corresponding optimal transport cost $D_{\varepsilon}$  be defined by
\begin{align*}
D_{\varepsilon}(P,Q)=  \inf_{\pi\in\Pi(P,Q)}\Ebb_{(X_1,X_2)\sim \pi}\big[c_{\varepsilon}(X_1,X_2) \big] .
\end{align*}
%From \cite{svm2008},
The minimum value of the natural classification risk is given by
\begin{align*}
\cR_{\operatorname{class},P}^{\star}= \inf_{f \text{ measurable}}\cR_{\operatorname{class},P}(f)
 =\Ebb\big[\min\{\eta(X),1-\eta (X)\}\big],
\end{align*}
which is reached when $f$ is the Bayes
%classification function
classifier, i.e.,    $f(\bx)=\sign(2\eta(\bx)-1)$ \citep{svm2008}.
%From \cite[Theorem 6.2]{Pydi2021},
The minimum value of the adversarial classification risk
% is
%$
%\widetilde\cR_{\operatorname{class},P}^{\star}(\varepsilon)= \inf_{f\text{ measurable}}\widetilde\cR_{\operatorname{class},P}(f,\varepsilon).
%$
%It can be verified that this
can be expressed as
\begin{displaymath}\label{eq_optimal}
\widetilde\cR_{\operatorname{class},P}^{\star}(\varepsilon)= \inf_{f\text{ measurable}}\widetilde\cR_{\operatorname{class},P}(f,\varepsilon)=\frac{1}{T+1}	\big[1- \inf_{Q\in\cP(\cX):Q\preceq T P_{0}}D_{\varepsilon}(Q,P_{1})\big],
\end{displaymath}
%\begin{equation}\label{eq_optimal}
%\begin{split}
%\widetilde\cR_{\operatorname{class},P}^{\star}(\varepsilon)&= \inf_{f\text{ %measurable}}\widetilde\cR_{\operatorname{class},P}(f,\varepsilon)
% \\&=\frac{1}{T+1}	\big[1- \inf_{Q\in\cP(\cX):Q\preceq T P_{0}}D_{\varepsilon}(Q,P_{1})\big],
%\end{split}
%\end{equation}
where $P_1=P_{X|Y=1}, P_{0}=P_{X|Y=-1},$ and $T=P(Y=-1)/P(Y=1)$ \citep[Theorem 6.2]{Pydi2021}.

The natural classification loss   and its adversarial  counterpart are
non-smooth and non-convex. Many surrogate losses have been considered in the context of standard classification. We specifically focus on  margin-based loss, where    a margin loss function $\phi$ exists such that the loss function satisfies    $\ell(u,y)=\phi(uy), \ (u,y)\in \Rbb\times \{-1,1\}.$
In general, the margin loss is selected to have a property called
consistency, which is satisfied by  a large family of convex losses \citep{svm2008}.  However, the adversarial version of these margin losses may not show the same consistency properties with respect to the adversarial classification loss. Moreover, \cite{meunier2022towards}  showed that no convex margin loss can be calibrated  in the adversarial setting. Consequently, it is challenging to study consistency in the general adversarial setting.

Let $C_{\operatorname{class} }(\eta,\bx,f)=\1 \{f(\bx)< 0\} \eta +\1 \{f(\bx)\geq  0\} (1-\eta)$ and  $C_{\operatorname{class}}^{\star}(\eta,\bx) =\min\{\eta, 1-\eta\}$. Let $C_{\phi}(\eta,\bx,f)=\phi(f(\bx) )  \eta  + \phi(-f(\bx))(1-\eta )$ and $C_{\phi}^{\star}(\eta,\bx)= \inf_{\alpha}  \phi(\alpha) \eta + \phi(-\alpha) (1-\eta ).$  Define $\cR_{P}^{\star}=\inf_{f\text{ measurable}}\cR_{P}(f)$.
\begin{assumption}\label{assumphi1}
	For any $\eta\in[0,1],$ $\bx\in\cX$ and measurable function $f$, 	
	\[C_{\phi}(\eta,\bx,f) -C_{\phi}^{\star}(\eta,\bx ) \geq a(C_{\operatorname{class} }(\eta,\bx,f) -C_{\operatorname{class}}^{\star}(\eta,\bx))\] holds for a positive constant $a$.
\end{assumption}
\begin{assumption}\label{assumphi2}
	There exist   positive constants  $c$   and $b$ such that \[\phi(0)  -C_{\phi}^{\star}(\eta,\bx)\geq b(1-C_{\operatorname{class}}^{\star}(\eta,\bx)) \] when $|\eta  -1/2|>c.$
\end{assumption}
Assumptions \ref{assumphi1} and \ref{assumphi2} can be satisfied by some common margin losses, such as the hinge loss.
\begin{corollary}\label{dd}
	Suppose the  conditions of Theorem \ref{thm:main1}  are satisfied and  $\phi$ is a continuous decreasing  margin function satisfying Assumptions \ref{assumphi1} and \ref{assumphi2}.
	Assume $|\eta(\bx)  -1/2|>c$ for any $\bx\in\cX$ and $\inf_{f\in\cH^{\alpha}} \cR_{P}(f )=  \cR_{P}^{\star}.$
	Then,
	\[
	\widetilde\cR_{\operatorname{class},P}(\widehat{f}_n,\varepsilon) -\widetilde\cR_{\operatorname{class},P}^{\star}(\varepsilon)\lesssim  n^{-\alpha/(2d+3\alpha)}\log n^{\xi}+\varepsilon.
	\]
\end{corollary}
For the case  where there might be $\eta(\bx)=1/2$,   we show that the natural classification risk of the adversarial estimator  converges to $ \cR_{\operatorname{class},P}^{\star}$ when $\varepsilon$ goes to $ 0.$
\begin{corollary}
	Suppose the  conditions of Theorem \ref{thm:main1}   are satisfied and $\phi$ is a   margin function satisfying Assumptions \ref{assumphi1}. Assume $\inf_{f\in\cH^{\alpha}} \cR_{P}(f )=  \cR_{P}^{\star}.$ Then, \[
	\cR_{\operatorname{class},P}(\widehat{f}_n) -  \cR_{\operatorname{class},P}^{\star} \lesssim  n^{-\alpha/(2d+3\alpha)}\log n^{\xi}+\varepsilon.\]
\end{corollary}

\subsection{Regression}\label{sec_reg}
Consider a nonparametric regression model
\begin{equation}\label{model_re}
Y=f_0(X)+\eta,
\end{equation}
where $Y\in\cY\subseteq[-1,1]$ is a response, $X\in\cX\subseteq[0,1]^d$ is a $d$-dimensional covariate vector,  $f_0\in\cH^{\alpha}$ is an unknown regression function, and  $\eta$ is an unobservable error satisfying $\Ebb(\eta| X)=0$ and $\Ebb(\eta^2)<\infty.$
% Let $(X_1,Y_1),\dots,(X_n,Y_n)$ be i.i.d. observations from model \eqref{model_re}.
Under model \eqref{model_re} and the
%squares
quadratic loss $\ell(u,y)= (u-y)^2,$  we denote the corresponding adversarial  estimator \eqref{eq:esti} by $\widehat{f}_n^{ls}$.
%{\color{blue}
%We further require that $\|\widehat{f}_n^{ls}\|_{\infty}\leq M$ for a sufficiently large constant $M>0.$ %Such requirement is reasonable since $\cY\subseteq[-1,1]$ and can be   satisfied by adding an additional %clipping layer   after the original output layer of the network.  In this case,   the loss function has a %bounded Lipschitz constant.}
We %aim at measuring
measure  the distance between   $\widehat{f}_n^{ls}$ and $f_0$
%under
using  the $L_2(P)$ norm  $\|\cdot\|_2:=\|\cdot\|_{L_2(P_X)}$, that is,  $\|f\|_2=\sqrt{\Ebb|f(X)|^2}$.
First, we derive a new error bound for the adversarial excess risk when using the quadratic loss.
\begin{theorem}\label{thm_linear}
	Suppose $\cH^{\alpha}$  with  $\alpha=r+\beta\geq 1$, where  $r\in\mathbb{N}_0$ and $\beta\in(0,1]$. Let $\gamma=\lceil \log_2 (d+r)\rceil$. Suppose Assumption  \ref{ass1}  holds.
	% There exists $c>0$ such that for any
	Suppose  $W\geq c(K/\log^{\gamma}K)^{(2d+\alpha)/(2d+2)}$ for a constant
	$c > 0$ and $L\geq 4\gamma+2.$
	If we select $
	K\asymp n^{(d+1)/(2d+5\alpha)}
	$ and $
	WL \asymp n^{(2d+\alpha)/(4d+10\alpha)} $, then for any adversarial  estimator  $\widehat{f}_n^{ls}$ satisfying $\|\widehat{f}_n^{ls}\|_{\infty}\leq M$ for a sufficiently large constant $M>0$, we have
	\begin{align*}
	\tE(\widehat{f}_n^{ls},\varepsilon)\lesssim  n^{-2\alpha/(2d+5\alpha)}\log n^{\lambda}+n^{(d+1)/(2d+5\alpha)}\varepsilon,
	\end{align*}
	where $\lambda=\max\{1,2\gamma\alpha/(d+1)\}.$
\end{theorem}
%The result of
Theorem \ref{thm_linear} shows that  the error rate of the adversarial excess risk of $\widehat{f}_n^{ls}$ can reach   $n^{-2\alpha/(2d+5\alpha)}$  up to a logarithmic factor when %the adversarial attack level satisfies
$\varepsilon=O(n^{-(d+2\alpha+1)/(2d+5\alpha)})$. It improves the rate $n^{-\alpha/(2d+3\alpha)}$ given  by Theorem \ref{thm:main1}. This is because a better control of the approximation error can be obtained with the quadratic loss function.
%By Lemma \ref{lem:Req}, we have
%%\begin{equation*}\label{eq:RP}
%%\begin{split}
%$ \cE(\widehat{f}_n^{ls})
%%	\\&=\Ebb(\{\widehat{f}_n^{ls}(X)-f_0(X)-\eta\}^2)-\Ebb(\eta^2)
%\leq \tE(\widehat{f}_n^{ls},\varepsilon)+\Lip^1(\ell) \varepsilon.$
%%\\&	\lesssim	n^{-\frac{\alpha}{2(d+\alpha+1) }} (\log n)^{\xi}+n^{\frac{d+1}{2(d+\alpha+1)}}\varepsilon
%%\end{split}
%%\end{equation*}
%Hence, an
%%error
%upper  bound for $ \cE(\widehat{f}_n^{ls})$ can be derived based on  Theorem \ref{thm_linear}.
%%According to the property
%Since $\cE(f)=\|f-f_0\|_2^2$ for the quadratic loss,
We   also obtain the convergence rate of  $\|\widehat{f}_n^{ls}-f_0\|_2.$   %follows.
% The formal result is  shown as follows.
\begin{corollary}
	%When
	Suppose the
	%assumptions and
	conditions of Theorem \ref{thm_linear} are satisfied, then
		\begin{align*}	
	\|\widehat{f}_n^{ls}-f_0\|^2_2
	%	\\&=\Ebb(\{\widehat{f}_n^{ls}(X)-f_0(X)-\eta\}^2)-\Ebb(\eta^2)
	\lesssim   n^{-2\alpha/(2d+5\alpha)}\log n^{\lambda}+n^{(d+1)/(2d+5\alpha)}\varepsilon.	
		\end{align*}
	If further $\varepsilon=O(n^{-(d+2\alpha+1)/(2d+5\alpha)})$, then
	%	\begin{align*}
	\[
	\|\widehat{f}_n^{ls}-f_0\|^2_2
	%	\\&=\Ebb(\{\widehat{f}_n^{ls}(X)-f_0(X)-\eta\}^2)-\Ebb(\eta^2)
	\lesssim   n^{-2\alpha/(2d+5\alpha)}\log n^{\lambda}.
	\]
	%\end{align*}
	%where $\lambda=\max\{1,2\gamma\alpha/(d+1)\}.$
\end{corollary}

\section{Related work}\label{sec_related}
%\subsection{Generalization error}
There is a line of
%interesting
%research
work focusing on the analysis of the Rademacher complexity of adversarial loss function class
\citep{Yin19b,khim2018adversarial,Awasthi20a,mustafa22a}.
Specifically,
\cite{Yin19b} investigated the adversarial Rademacher complexity of the linear models under perturbations measured in $\ell_{\infty}$ norm.
%, where the coefficients were under $\ell_p$ norm constraint and the  perturbations were measured in $\ell_{\infty}$ norm.
%Later,
%\cite{khim2018adversarial}  investigated the linear models with the perturbations measured in
%a  general $\ell_r$ norm and the coefficients  under the $\ell_2$ and $\ell_{r^{\star}}$ norm constraints, where $1/r+1/r^{\star}=1$.
The result    was later generalized  by \cite{khim2018adversarial} and \cite{Awasthi20a} to the cases where the perturbations were measured in a general $\ell_r$ norm.
% and the coefficients were under a general $\ell_p$ norm constraint.
% \cite{Awasthi20a} also suggested the selection of  the order of constrained norm $p$ to obtain statistical %efficiency.
%{\color{blue}
%	For
%	[[[the case of ]]]
%	nonparametric models using neural networks, \cite{Yin19b, Awasthi20a} investigated the adversarial training with one-hidden-layer neural network.}
{\color{black}For neural network models, \cite{Yin19b, Awasthi20a} investigated adversarial training
	when the model was a
	%one-hidden-layer
	neural network with a single hidden layer.}
%Specifically,
%\cite{Yin19b} replaced the adversarial loss by a surrogate adversarial loss based on the SDP relaxation  \cite{raghunathan2018certified} and analyzed the Rademacher complexity of the  surrogate.  %\cite{Awasthi20a} investigated the adversarial loss function class directly.
In \cite{khim2018adversarial} and  \cite{mustafa22a},    deep neural networks were studied.  \cite{khim2018adversarial} proposed a tree transformation
%such that the transformed loss function
to upper bound the adversarial loss function.
% and analyzed the Rademacher complexity of the transformed function class.
In \cite{mustafa22a},    the covering number of the adversarial loss function class was shown to be upper bounded by the covering number of a newly defined loss function class over an extended training set. The Rademacher complexity of the adversarial loss function class can then be obtained by analyzing the covering number of the new loss function class.
\cite{Lee2018} and \cite{Tu2019} transformed the adversarial learning problem into a DRO problem.
% and
%aimed to
%derived an upper bound of the local worst-case risk.
%Based on the dual form of the DRO and a more refined  analysis, \cite{Tu2019} improved the results %of \cite{Lee2018}.
However, in all the aforementioned works, the authors did not consider the approximation error  and only focused on the generalization error, see Table \ref{table_compare}.
% for detail.

%which was demonstrated to be insufficient to evaluate the generalization performance of   adversarial estimator. We addressed the problem by investigating the adversarial excess risk, which gave a more comprehensive evaluation  of the  generalization performance.    In addition,  our proposed method was demonstrated to be applicable to general $\ell_r$ attack and a large family of estimation function classes, including the deep neural network. The above discussion is briefly summarized in

%\subsection{Trade-off between accuracy and adversarial robustness}
The problem of the trade-offs between accuracy and adversarial robustness has been studied recently \citep{madry2018towards,Schmidt2018data,tsipras2018robustness,
Duchi2019adversarial,dobriban2020provable,nma2022on,Xing2021,Dan2020}.
% was still limited.
%\cite{tsipras2018robustness, zhang19p} illustrated  examples where a classifier that was both optimal robust and  accurate did not exist even in the infinite data limit. It suggested that the trade-off might be an inherent
%% trait
%property of the data generating distribution.  In other words, there was a fundamental
%conflict between accuracy and adversarial robustness.   \cite{raghunathan2018certified} constructed a counterexample, where the optimal predictor with an infinite amount of data performed well in terms of both accuracy and adversarial robustness, however, a trade-off was still observed with finite
%data.
%%Their experiments suggested that the trade-off was due to insufficient samples.
The works of \cite{Javanmard_classification} and \cite{Javanmard_linear} gave a precise theoretical characterization of the trade-offs in the linear regression and parametric binary classification problems under the Gaussian assumption.
%and worked on a regime where
%and $n/d\rightarrow \delta\in(0,\infty)$.
%In both setting  they  showed  the existence of fundamental trade-off   by analyzing the Pareto optimal curve.
%of a two-dimensional region
%consisting of all the achievable  pairs of natural risk and adversarial risk.
For the adversarial training, \cite{Javanmard_linear} characterized its trade-off curve by calculating the natural risk and adversarial risk of the adversarial estimator that was derived from different adversarial attack levels.
%%When $\delta$ was large,
They found that
%%the trade-off curve approached the Pareto optimal curve. It demonstrated
the adversarial training hurt the accuracy if robustness is pursued.
%in the linear regression.
%In the binary classification problem, \cite{Javanmard_classification} focused on the linear classifier and
% characterized the trade-off curve for the adversarial estimator.
%However, \cite{Javanmard_linear, Javanmard_classification}  focused on specific linear settings and the analysis methods relied   on the normal assumption.
However, there is still a lack of systematic
%a  comprehensive
theoretical understanding of the trade-offs in  general nonparametric settings.
% was still limited.

 \section{Conclusions}\label{sec_con}
 In this paper, we have proposed a general approach
 %analysis method
 to evaluating the
  generalization performance of
 %the
 estimators based on adversarial training under misspecified models.
 %To guarantee that the adversarial risk is well-defined in our studied setting, we show that it is equal to a %natural risk, whose expectation is taken over a shifted distribution.
 %By evaluating the distance between the shifted distribution and the data generating distribution,
 %We demonstrate  have shown that
 %To ensure  guarantee
 % that the adversarial risk
 %that the adversarial training procedure is well-defined,
 The adversarial risk is shown to be equivalent to the risk induced by a distributional adversarial attack under certain smoothness conditions. This shows that the adversarial training procedure is well-defined.
 %To evaluate the generalization performance of the adversarial estimator,
 %we study the adversarial excess risk.
 %Our proposed analysis method includes investigations on both
 %generalization error and approximation error.  Further,
 We have established non-asymptotic error bounds on the adversarial excess risk,
 % are established.  The error bounds %could
 which  achieve the rate $O(n^{-\alpha/(2d+3\alpha)})$ up to a logarithmic factor for a
 %general
 Lipschitz loss function and  can be improved to $O(n^{-2\alpha/(2d+5\alpha)})$ up to a logarithmic factor  when using
 % the square
 the quadratic loss.
 %also measure the performance of the adversarial estimator under other risks and
 %have also provided a theoretical understanding of the trade-offs between  accuracy and adversarial robustness.
 We have also theoretically demonstrated that adversarial robustness can hurt accuracy in a general nonparametric setting.
 % in general setting. It demonstrates  the loss of accuracy  when the  adversarial robustness reaches a certain %level.
 %We apply the non-asymptotic error bounds to different problems and demonstrate the applicability and %extensibility of the proposed method. It shows that our theoretical analysis method is general and can be %easily extended to various settings.

% This work has some limitations.
 There are several interesting problems that deserve further study.
 First, the Lipschitz type condition (e.g. Assumption \ref{ass2}) plays an important role
 in our analysis. However, this assumption is not satisfied for the important quadratic
 loss function when the support of $Y$ is unbounded. It would be interesting to relax
 this assumption in the future.
 %we only considered the upper bounds for the adversarial and usual risks.
 %Another limitation of this work is that
 Also, we have only considered robustness against adversarial examples in $X$. How to generalize the results to the case when there are also adversarial examples in both $X$ and $Y$
 is an important problem for future work. Finally, a complete analysis of the trade-offs between accuracy and adversarial robustness requires the establishment of lower bounds for the adversarial excess risk.

\iffalse
{\color{blue}\section*{Acknowledgments}
We would like to acknowledge the assistance of volunteers in putting
together this example manuscript and supplement.
}
\fi

\bigskip
\appendix
%\begin{center}
%\noindent
\textbf{%\large
Appendix.}
%\end{center}
%\medskip
In the Appendix, we give the proofs of the results stated in the paper
and provide additional technical details needed in the proofs.
%are give in the Supplementary Materials.

\section{Auxiliary lemmas}

Theorem 12.2 of \cite{anthony1999neural} showed  an upper bound for the uniform covering number of a function class based on its pseudo-dimension. The result is given   below.
\begin{lemma}\label{lem_cover}
	Let $\cF$ be a set of real functions from a domain $\cX$ to the bounded interval $[0, B]$. Let $\varepsilon>0$ and denote the pseudo-dimension of $\cF$ by $\operatorname{Pdim}(\cF)$. When $n \geq \operatorname{Pdim}(\cF)$, then the uniform covering number satisfies
	$$
	\mathcal{N}_{\infty}(\varepsilon, \cF, n) \leq \left(\frac{enB}{\varepsilon \operatorname{Pdim}(\cF)}\right)^{\operatorname{Pdim}(\cF)}.
	$$
\end{lemma}

Let the size $S$ of a FNN be defined as the total number of parameters in the network.
For the class $\cN\cN(W,L)$, its parameter satisfy the simple relationship $\max\{W,L\}\leq S \leq c  W^2L$ for a constant $c.$ Based on Theorems 7 of \cite{JMLRPeter2019}, an  upper bound  for the  pseudo-dimension of  $\cN\cN(W,L)$  follows  as below.
\begin{lemma}\label{lem_psuedo}
	There exists universal constant $C$ such that
	\[
	\operatorname{Pdim}(\cN\cN(W,L))  \leq C \cdot  W^2 L^2 \log (W^2L).
	\]
\end{lemma}

%
%The following two result are from Proposition 2.1 and Lemma 2.3 in \cite{jiao2022approximation}.
%\begin{lemma}\label{lem:SNN}
%	 $\mathcal{SNN}(W,L,K)\subseteq\mathcal{NN}(W,L,K)\subseteq \mathcal{SNN}(W_1,L,K).$
%\end{lemma}
%\begin{lemma}\label{lem:com_SNN}
%	For any $\boldsymbol{x}_{1}, \ldots, \boldsymbol{x}_{n} \in[-B, B]^{d}$ with $B \geq 1$, let $S:=\left\{\left(\phi\left(\boldsymbol{x}_{1}\right), \ldots, \phi\left(\boldsymbol{x}_{n}\right)\right): \phi \in\right.$ $\left.\mathcal{S} \mathcal{N N}_{d, 1}(W, L, K)\right\} \subseteq \mathbb{R}^{n}$, then
%	$$
%	\mathcal{R}_{n}(S) \leq \frac{2}{n} K \sqrt{L+2+\log (d+1)} \max _{1 \leq j \leq d+1} \sqrt{\sum_{i=1}^{n} x_{i, j}^{2}} \leq \frac{2 B K \sqrt{L+2+\log (d+1)}}{\sqrt{n}},
%	$$
%	where $x_{i, j}$ is the $j$-th coordinate of the vector $\tilde{\boldsymbol{x}}_{i}=\left(\boldsymbol{x}_{i}^{\top}, 1\right)^{\top} \in \mathbb{R}^{d+1}$.
%\end{lemma}
The following result is from 	Corollary 4.2.13 of \cite{vershynin2018}.
\begin{lemma}[Covering numbers of the Euclidean ball]\label{lem_coverball}
	The covering numbers of the unit Euclidean ball $B_2^d=\{\bx\in\Rbb^{d}:\|\bx\|\leq 1\}$ satisfy the following for any $\varepsilon\in (0,1]$.
	$$
	\left(\frac{1}{\varepsilon}\right)^d \leq \mathcal{N}\left(B_2^d, \varepsilon\right) \leq\left(\frac{3}{\varepsilon}\right)^d,
	$$
	where $\mathcal{N} (B_2^d, \varepsilon)$ denotes the minimal number of
	balls of radius $\varepsilon$ needed to cover  $B_2^d$ under the Euclidean norm.
\end{lemma}

The following result   is from   \cite{srebro2010note,liu2021wasserstein}
\begin{lemma}[Random Covering Entropy Integral]\label{lem:cover}
	Suppose the function class $\cF$ is defined over $\cX$ and satisfies $\sup_{f\in\cF}\|f\|_{\infty}\leq B$. For any samples $\bx_1,\dots,\bx_n$ from $\cX$, we have
	\begin{equation*}
	\begin{split}
	\Ebb_{\boldsymbol{\sigma}}\Big\{\sup_{f\in\cF }  \frac{1}{n}\sum_{i=1}^n\sigma_if( \bx_i)  \Big\}
	\leq \inf_{\delta\geq 0}\Big\{
	4\delta+12\int_{\delta}^{B}\sqrt{\frac{\log\mathcal{N}(u,\cF,L_2(P_n))}{n}}du
	\Big\},
	\end{split}
	\end{equation*}
	where $\boldsymbol{\sigma}=(\sigma_1,\dots, \sigma_n)$  consists   of  i.i.d. the Rademacher variables and is  independent from the samples and $L_2(P_n)$ denotes the data dependent $L_2$ metric.
\end{lemma}

The following is from \cite[Lemma A.3.18]{svm2008} and \cite{castaing2006convex}.
\begin{lemma}[Aumann's measurable selection principle]\label{lem:Au}
	Let $(X, \mathcal{A})$ be a complete measurable space, $Z$ be a Polish space equipped with its Borel $\sigma$-algebra, and $Y$ be a measurable space. Furthermore, let $h: X \times Z \rightarrow Y$ be a measurable map, $A \subset Y$ be measurable, and $F $ be defined by $$
	\begin{aligned}
	F: X & \rightarrow 2^Z \\
	x & \mapsto\{z \in Z: h(x, z) \in \cA\}
	\end{aligned}
	$$
	where $2^Z$ denotes the set of all subsets of $Z$.  Let $\varphi: X \times Z \rightarrow[0, \infty]$ be measurable and $\psi: X \rightarrow[0, \infty]$ be defined by
	\begin{equation}\label{eq:inf}
	\psi(x)=\inf _{z \in F(x)} \varphi(x, z), \quad x \in X .
	\end{equation}
	Then $\psi$ is measurable. Define $\operatorname{Dom} F=\{x\in X:F(x)\ne\emptyset\}.$
	If the infimum in \eqref{eq:inf} is attained for all $x\in \operatorname{Dom} F$, then there exists a measurable function $f^{\star}: X \rightarrow Z$ with $ f^{\star}(x)\in F(x)$ and $\psi(x)=\varphi\left(x, f^*(x)\right)$ for all $x \in \operatorname{Dom} F$.
\end{lemma}
%The following result is given in \cite{vanasymp} on page 268
%\begin{lemma}\label{lem:ecdf}
%	If $X_1,X_2,\dots, X_n$ are i.i.d. random variables with distribution function $F$. Let $F_n(t):=\frac{1}{n}\sum_{i=1}^{n}\1 \{X_i\leq t\}$ be the empirical distribution function. Then,
%	\[
%\limsup_{n\rightarrow\infty}\sqrt{\frac{n}{2\log \log n}}	\|F_n-F\|_{\infty}\leq \frac{1}{2}, \quad a.s.,
%	\]
%	with equality if $F$ takes  on the value $\frac{1}{2}.$
%\end{lemma}
\begin{lemma}\label{lem_ineqaulity}
	Suppose $\phi$ is a continuous decreasing  margin function  and there exist     constants $a,c>0$ such that for any  $\bx\in\cX$ and continuous function $f$ we have $C_{\phi}(\eta,\bx,f) -C_{\phi}^{\star}(\eta,\bx ) \geq a(C_{\operatorname{class} }(\eta,\bx,f) -C_{\operatorname{class}}^{\star}(\eta,\bx))$ for any $\eta\in[0,1]$,  and
	$\phi(0)  -C_{\phi}^{\star}(\eta,\bx)\geq a(1-C_{\operatorname{class}}^{\star}(\eta,\bx)) $
	when $|\eta  -1/2|>c$. 	Let  $\eta(\bx)=P(Y=1|X=\bx)$. Suppose $|\eta(\bx)  -1/2|>c$ for any $\bx\in\cX$. For any  continuous function $f$, we have
	\begin{align*}
	\widetilde\cR_{P}(f,\varepsilon) - \inf_{f\in\cH^{\alpha}} \widetilde\cR_{P}(f,\varepsilon)  \geq  a  (\widetilde\cR_{\operatorname{class},P}(f,\varepsilon) - \cR_{\operatorname{class},P}^{\star})  - \inf_{f\in\cH^{\alpha}} \widetilde\cR_{P}(f,\varepsilon) + \cR_{P}^{\star}.
	\end{align*}
\end{lemma}
\begin{proof}[Proof of Lemma \ref{lem_ineqaulity}]
	For the   classification risk,
	\begin{align*}
	\cR_{\operatorname{class},P}(f)&=
	\Ebb_{(X,Y)\sim P}  \1 \{\sign f(X)\ne Y \}
	\\&=\int \1{\{f(\bx)< 0\}} \eta(\bx) +\1{\{f(\bx)\geq  0\}} (1-\eta(\bx) )d P_{X}
	\\&=\int  C_{\operatorname{class}}(\eta(\bx),\bx,f)d P_{X},
	\end{align*}
	where $C_{\operatorname{class}}(\eta,\bx,f)=\1{\{f(\bx)< 0\}}\eta +\1{\{f(\bx)\geq  0\}} (1-\eta)$. For the adversarial classification risk,
	\begin{align*}
	\widetilde\cR_{\operatorname{class},P}(f,\varepsilon)&=
	\Ebb_{(X,Y)\sim P} \big[\sup_{ X^{\prime} \in B_{\varepsilon}(X) }  \1{\{\sign f(X^{\prime} )\ne Y \} }\big]
	\\&=\int \sup_{\bx^{\prime}\in B_{\varepsilon}(\bx)}\1{\{f(\bx^{\prime})< 0\}} \eta(\bx) +\sup_{\bx^{\prime}\in B_{\varepsilon}(\bx)}\1{\{f(\bx^{\prime})\geq  0\}} (1-\eta(\bx) )d P_{X}
	\\&=\int C_{\operatorname{class},\varepsilon}(\eta(\bx),\bx,f)d P_{X},
	\end{align*}
	where $C_{\operatorname{class},\varepsilon}(\eta,\bx,f) =\sup_{\bx^{\prime}\in B_{\varepsilon}(\bx)}\1{\{f(\bx^{\prime})< 0\}} \eta +\sup_{\bx^{\prime}\in B_{\varepsilon}(\bx)}\1{\{f(\bx^{\prime})\geq  0\}} (1-\eta )$. For the surrogate loss $\phi$, the surrogate risk and adversarial surrogate risk are
	\begin{align*}
	\cR_{P}(f)&=\Ebb_{(X,Y)\sim P} [\phi(f(X)Y)]
	\\&=\int \phi(f(\bx) )  \eta(\bx) + \phi(-f(\bx))(1-\eta(\bx) )d P_{X}
	\\&=\int  C_{\phi}(\eta(\bx),\bx,f)d P_{X},
	\end{align*}
	where $C_{\phi}(\eta,\bx,f)=\phi(f(\bx) )  \eta  + \phi(-f(\bx))(1-\eta )$, and
	\begin{align*}
	\widetilde\cR_{P}(f,\varepsilon)&=\Ebb_{(X,Y)\sim P} [\sup_{X^{\prime} \in B_{\varepsilon}(X) }\phi(f(X^{\prime})Y)]
	\\&=\int\sup_{ \bx^{\prime} \in B_{\varepsilon}(\bx) }  \phi(f(\bx^{\prime}) )  \eta(\bx) + \sup_{ \bx^{\prime} \in B_{\varepsilon}(\bx) }\phi(-f(\bx^{\prime}))(1-\eta(\bx) )d P_{X}
	\\&=\int  C_{\phi_,\varepsilon}(\eta(\bx),\bx,f)d P_{X},
	\end{align*}
	where $C_{\phi,\varepsilon}(\eta,\bx,f)=\sup_{\bx^{\prime} \in B_{\varepsilon}(\bx) }  \phi(f(\bx^{\prime}) )  \eta  + \sup_{\bx^{\prime} \in B_{\varepsilon}(\bx) }\phi(-f(\bx^{\prime}))(1-\eta  )$.

	Based on Proposition 2.2 of \cite{meunier2022towards}, we have
	\[
	C_{\operatorname{class},\varepsilon}^{\star}(\eta,\bx ):=\inf_{f \text{ measurable}}C_{\operatorname{class},\varepsilon}(\eta,\bx,f) =\inf_{\alpha}  \1{\{\alpha< 0\}} \eta +
	\1{\{\alpha\geq  0\}} (1-\eta ),
	\]
	and
	\[
	C^{\star}_{\phi,\varepsilon}(\eta,\bx):=\inf_{f \text{ measurable}}C_{\phi,\varepsilon}(\eta,\bx,f) =\inf_{\alpha}  \phi(\alpha) \eta + \phi(-\alpha) (1-\eta ).%=\min\{\eta, 1-\eta\}.
	\]
	Then, $	C_{\operatorname{class},\varepsilon}^{\star}(\eta,\bx )=C_{\operatorname{class}}^{\star}(\eta,\bx)$ and $C^{\star}_{\phi,\varepsilon}(\eta,\bx)=C^{\star}_{\phi}(\eta,\bx)$.	
	For any continuous function $f$, let $a_{f,\varepsilon}=\sup_{\bx^{\prime} \in B_{\varepsilon}(\bx)}f(\bx^{\prime})$ and $b_{f,\varepsilon}=\inf_{\bx^{\prime} \in B_{\varepsilon}(\bx)}f(\bx^{\prime}).$
	\begin{align*}
	&C_{\operatorname{class},\varepsilon}(\eta,\bx,f) -C_{\operatorname{class},\varepsilon}^{\star}(\eta,\bx )\\&=\sup_{\bx^{\prime}\in B_{\varepsilon}(\bx)}\1{\{f(\bx^{\prime})< 0\}} \eta +\sup_{\bx^{\prime}\in B_{\varepsilon}(\bx)}\1{\{f(\bx^{\prime})\geq  0\}} (1-\eta )-C_{\operatorname{class},\varepsilon}^{\star}(\eta,\bx )
	\\&= \1{\{b_{f,\varepsilon}< 0\}} \eta + \1{\{a_{f,\varepsilon}\geq  0\}} (1-\eta )-\min\{\eta, 1-\eta\}.
	\end{align*}
	If $a_{f,\varepsilon},b_{f,\varepsilon}\geq 0$ or $a_{f,\varepsilon},b_{f,\varepsilon}<0$, we have
	\begin{align*}
	C_{\operatorname{class},\varepsilon}(\eta,\bx,f) -C_{\operatorname{class},\varepsilon}^{\star}(\eta,\bx )= C_{\operatorname{class} }(\eta,\bx,f) -C_{\operatorname{class}}^{\star}(\eta,\bx).
	\end{align*}
	If $a_{f,\varepsilon}\geq 0$  and $b_{f,\varepsilon}<0$,
	\begin{align*}
	C_{\operatorname{class},\varepsilon}(\eta,\bx,f) -C_{\operatorname{class},\varepsilon}^{\star}(\eta,\bx )=1-C_{\operatorname{class}}^{\star}(\eta,\bx ).
	\end{align*}
	For the adversarial surrogate loss, since $\phi$ is a decreasing continuous function, we have
	\begin{align*}
	&C_{\phi,\varepsilon}(\eta,\bx,f) -C_{\phi,\varepsilon}^{\star}(\eta,\bx)
	\\&=\sup_{\bx^{\prime} \in B_{\varepsilon}(\bx) }  \phi(f(\bx^{\prime}) )  \eta  + \sup_{\bx^{\prime} \in B_{\varepsilon}(\bx) }\phi(-f(\bx^{\prime}))(1-\eta)-C_{\phi,\varepsilon}^{\star}(\eta,\bx)
	\\&=\phi(b_{f,\varepsilon})\eta + \phi(-a_{f,\varepsilon})(1-\eta )-C^{\star}_{\phi,\varepsilon}(\eta,\bx)
	\\&\geq C_{\phi}(\eta,\bx,f) -C_{\phi,\varepsilon}^{\star}(\eta,\bx).
	\end{align*}
	If $a_{f,\varepsilon},b_{f,\varepsilon}\geq 0$ or $a_{f,\varepsilon},b_{f,\varepsilon}<0$, we have
	\begin{align*}
	C_{\phi,\varepsilon}(\eta,\bx,f) -C_{\phi,\varepsilon}^{\star}(\eta,\bx)\geq a (C_{\operatorname{class},\varepsilon}(\eta,\bx,f) -C_{\operatorname{class},\varepsilon}^{\star}(\eta,\bx) ).
	\end{align*}
	If $a_{f,\varepsilon}\geq 0$  and $b_{f,\varepsilon}<0$, when $|\eta -\frac{1}{2}|>c$, then
	\begin{align*}
	C_{\phi,\varepsilon}(\eta,\bx,f) -C_{\phi,\varepsilon}^{\star}(\eta,\bx)\geq\phi(0)  -C_{\phi,\varepsilon}^{\star}(\eta,\bx)\geq  a (C_{\operatorname{class},\varepsilon}(\eta,\bx,f) -C_{\operatorname{class},\varepsilon}^{\star}(\eta,\bx)).
	\end{align*}
	Consequently,  for any continuous function $f$, any $\bx\in\cX$, if $|\eta -\frac{1}{2}|>c$, then
	\begin{align*}
	C_{\phi,\varepsilon}(\eta,\bx,f) -C_{\phi,\varepsilon}^{\star}(\eta,\bx) \geq  a  (C_{\operatorname{class},\varepsilon}(\eta,\bx,f) -C_{\operatorname{class},\varepsilon}^{\star}(\eta,\bx )).
	\end{align*}
	Since   $|\eta(\bx) -\frac{1}{2}|>c$ for any $\bx\in\cX$, we have
	\begin{align*}
	\widetilde\cR_{P}(f,\varepsilon) - \inf_{f \text{ measurable}}\cR_{P}(f) \geq  a  (\widetilde\cR_{\operatorname{class},P}(f,\varepsilon) - \inf_{f \text{ measurable}}\cR_{\operatorname{class},P}(f)).
	\end{align*}
	Recall $\cR_{P}^{\star}=\inf_{f \text{ measurable}}\cR_{P}(f) $ and $\cR_{\operatorname{class},P}^{\star}= \inf_{f \text{ measurable}}\cR_{\operatorname{class},P}(f)$. Then,
	\begin{align*}
	\widetilde\cR_{P}(f,\varepsilon) - \inf_{f\in\cH^{\alpha}} \widetilde\cR_{P}(f,\varepsilon) &= \widetilde\cR_{P}(f,\varepsilon) -\cR_{P}^{\star}- \inf_{f\in\cH^{\alpha}} \widetilde\cR_{P}(f,\varepsilon) +\cR_{P}^{\star}
	\\&\geq  a  (\widetilde\cR_{\operatorname{class},P}(f,\varepsilon) - \cR_{\operatorname{class},P}^{\star})  - \inf_{f\in\cH^{\alpha}} \widetilde\cR_{P}(f,\varepsilon) + \cR_{P}^{\star},
	\end{align*}
	which completes the proof.
\end{proof}

\section{Proofs of Lemmas 2.3 and 2.4}
\subsection{Proof of Lemma 2.3}
\begin{proof}[Proof of Lemma 2.3]
	Since $\ell, f$ are continuous and Assumption 2.1  hold, there exists a constant $M>0$ such that $|\ell(f(\bx), y)|\leq M$ for any $(\bx,y)\in\cZ.$ For any $\bz=(\bx,y), \bz^{\prime}=(\bx^{\prime},y^{\prime})\in\cZ$,  we define a function $\varphi(\cdot,\cdot):\cZ\times\cZ\mapsto [0, \infty)$ by
	\[
	\varphi(\bz,\bz^{\prime})=M- \ell(f(\bx^{\prime}), y^{\prime}).
	\]
	From the continuity of $\ell$ and $f$, the function $\varphi$ is continuous. Define a function  $h(\cdot,\cdot):\cZ\times\cZ\mapsto \Rbb$ by $h(\bz,\bz^{\prime})= \1\{\|\bx-\bx^{\prime}\|_{\infty}\leq \varepsilon, y=y^{\prime}\}$  and let $F:\cZ\mapsto 2^{\cZ}$ be defined by
	\[
	F(\bz)=\big\{\bz^{\prime}\in\cZ:h(\bz,\bz^{\prime})\in\{1\}\big\}.
	\]
	It  follows that $F(\bz)= \{(\bx^{\prime}, y): \|\bx-\bx^{\prime}\|_{\infty}\leq \varepsilon \}=B_{\varepsilon}(\bx)\times \{y\}.$ Let $\psi:\cZ\mapsto[0, \infty)$ be defined by $\psi(\bz)=\inf_{\bz^{\prime}\in F(\bz)}\varphi(\bz,\bz^{\prime})$, that is,
	\begin{align*}
	\psi(\bz)&= \inf_{\bx^{\prime}\in B_{\varepsilon}(\bx)}\big\{M- \ell(f(\bx^{\prime}), y )\big\}
	\\&=M-\sup_{\bx^{\prime}\in B_{\varepsilon}(\bx)} \ell(f(\bx^{\prime}), y).
	\end{align*}
	By the continuity of $\varphi,$ the infimum in the above can be attained for any $\bz\in\cZ$. Therefore, it  follows from Aumann's measurable selection principle in Lemma \ref{lem:Au} that $\psi(\bz)$ is measurable  and there exists a measurable function $K^{\star}:\cZ\mapsto\cZ$ with $\psi(\bz)=\varphi(\bz,K^{\star}(\bz) )$. In addition, from $K^{\star}(\bz)\in F(\bz), $  we can denote $ K^{\star}$ by $ K^{\star}(\bz)=(T^{\star}(\bz), y),$ where $T^{\star}$ is measurable function satisfying $T^{\star}(\bz)\in B_{\varepsilon}(\bx).$  Combining these together, we show
	\begin{align*}
	\sup_{\bx^{\prime}\in B_{\varepsilon}(\bx)} \ell(f(\bx^{\prime}), y)=\ell(f(T^{\star}(\bz)), y).
	\end{align*}
	Let the joint distribution of $(T^{\star}(Z), Y)$ be denoted by $P^{\star}$, it follows
	\begin{align*}
	\Ebb_{(X,Y)\sim P}[\sup_{X^{\prime}\in B_{\varepsilon}(X)} \ell(f(X^{\prime}), Y)]
	=\Ebb_{(X,Y)\sim P^{\star}}[\ell(f(X), Y)].
	\end{align*}	
	In addition, $W_\infty(P^{\star},P) \leq\esssup_{Z\sim P}d_{\cZ}(K^{\star}(Z),Z)=\esssup_{Z\sim P}\|T^{\star}(Z)-X\|_{\infty}\leq \varepsilon, $
	which completes   the proof.
	%    \[
	%   | \ell(f(\bx), y)|\leq |\ell(f(\bx), y)-\ell(f(\textbf{0}), 0)|+|\ell(f(\textbf{0}), 0)|\leq \Lip(\ell)(\Lip(f)+1)+|\ell(f(\textbf{0}), 0)|,
	%    \]
\end{proof}

\subsection{Proof of Lemma 2.4}
\begin{proof}[Proof of Lemma 2.4]
	Based on Lemma 2.3, with $Z=(X,Y)$, we show
	\begin{align*}
	\widetilde \cR_{P}(f,\varepsilon)-\cR_{P}(f)&= \Ebb \ell(f(T^{\star}(Z)), Y)- \Ebb \ell(f(X), Y)
	\\&\leq \Lip^1(\ell)\Ebb |f(T^{\star}(Z))-f(X)|
	\\&\leq \Lip^1(\ell)\Lip(f)\varepsilon.
	\end{align*}
	In addition,  from $\ell(f(\bx ), y)\leq \sup_{\bx^{\prime}\in B_{\varepsilon}(\bx)} \ell(f(\bx^{\prime}), y)$, we have
	\begin{equation*}
	\cR_{P}(f)\leq \widetilde \cR_{P}(f,\varepsilon).
	\end{equation*}
	Consequently, we show
	\begin{align*}
	\cR_{P}(f)\leq \widetilde \cR_{P}(f,\varepsilon)\leq \cR_{P}(f)+\Lip^1(\ell)\Lip(f)\varepsilon,
	\end{align*}
	which completes the proof.
\end{proof}

\section{Proofs of theorems and corollaries}
\subsection{Proof of Theorem 2.5}
\begin{proof}[Proof of Theorem 2.5]
	Based on  Lemma 2.3, we consider the joint distribution $P^{\star}$ of $(T^{\star}(Z), Y)$. 	
	%	Then,
	%   $
	%   	W_\infty(P^{\star},P) \leq\esssup_{Z\sim P}d_{\cZ}(K^{\star}(Z),Z)
	%    =\esssup_{Z\sim P}\|T^{\star}(Z)-X\|_{\infty}\leq \varepsilon,
	%   	$
	%	which follows
	%	\begin{align*}
	%		\widetilde \cR_{P}(f,\varepsilon)=  \cR_{P^{\star}}(f)\leq\sup_{Q:W_{\infty}(Q,P)\leq \varepsilon}\cR_{Q}(f).
	%	\end{align*}
	For any given $y\in\cY$, we study the conditional distributions $P_y$ and $P^{\star}_y$. Specifically, $P_y$ is the conditional distribution of $X$   given $Y=y$ and $P^{\star}_y$ is the conditional distribution of $T^{\star}(Z)$   given $Y=y$. Then, we have
	$$
	W_\infty(P^{\star}_y,P_y) \leq \esssup_{X\sim P_y}\|T^{\star}(X,y)-X\|_{\infty}\leq \varepsilon.
	$$
	It shows that $P^{\star}\in\Gamma_{\varepsilon}$. Therefore,
	\begin{equation*}
	\widetilde \cR_{P}(f,\varepsilon)=\cR_{P^{\star}}(f)\leq \sup_{Q\in\Gamma_{\varepsilon}}\cR_{Q}(f).
	\end{equation*}
	
	Next, we prove the reverse direction. For any $Q\in\Gamma_{\varepsilon}$, we denote   the corresponding pair of variables  by $(\widetilde{X},\widetilde{Y})$, and let the conditional distribution of $\widetilde{X}$ given $\widetilde{Y}=y$ be denoted by $Q_y.$ With a similar analysis method in the proof of  \cite[Theorem 1]{Pydi2021a}, the condition   $W_\infty(Q_y,P_y)\leq \varepsilon$ implies that there is a joint distribution $\pi$ on $\cX\times \cX$ such that the marginals are $Q_y$ and $P_y$ and  satisfies
	\[
	\esssup_{(\widetilde{X},X)\sim\pi}\|\widetilde{X} -X\|_{\infty}\leq \varepsilon.
	\]
	It follows that $\widetilde{X}\in B_{\varepsilon}(X)$ almost surely. This leads to
	\begin{align*}     	
	\Ebb_{X\sim P_{y}}\big[
	\sup_{X^{\prime}\in B_{\varepsilon}(X)} \ell(f(X^{\prime}), y)
	\big]-\Ebb_{\widetilde{X}\sim Q_{y}}\big[ \ell(f(\widetilde{X}), y)
	\big]\geq 0.
	\end{align*}
	Since $\widetilde{Y}$ follows $P_Y$, then $\widetilde \cR_{P}(f,\varepsilon)\geq  \cR_{Q}(f)$ holds for any $Q\in\Gamma_{\varepsilon}$.  Hence,
	\[
	\widetilde \cR_{P}(f,\varepsilon)\geq  \sup_{Q\in\Gamma_{\varepsilon}}\cR_{Q}(f).
	\]
	This together with the above result implies
	$
	\widetilde\cR_{P}(f,\varepsilon)=  \sup_{Q\in\Gamma_{\varepsilon}}\cR_{Q}(f)
	$, which completes our proof.
\end{proof}

\subsection{Proof of Theorem 3.1}
%{\bf Notation:} Suppose that $\cF$ is a class of real functions from the domain $\cZ$. Given a set of samples
%$\bz_{1:n}=\{\bz_1,\dots, \bz_n\}\in \cZ^n$, let $\cF|_{\bz_{1:n}}$ denote the set
%$\{(f(\bz_1),\dots, f(\bz_n)):f\in \cF\}.$
%Let $\cN (\varepsilon,\cF|_{\bz_{1:n}}, d)$ denote the covering number  for the set $\cF|_{\bz_{1:n}}$ with radius $\varepsilon$ under the metric $d$.
%When the metric takes the norm $\|\cdot\|_{\infty}$, let $\cN_{\infty}(\varepsilon,\cF,n)$ denote the uniform covering number, that is, $\cN_{\infty}(\varepsilon,\cF,n)=\max\{\cN (\varepsilon,\cF|_{\bz_{1:n}},\|\cdot\|_{\infty}):\bz_{1:n}\in \cZ^n\}$.
Let $f_0\in\argmin_{f\in    \cN\cN(W,L,K)} \widetilde\cR_{P}(f,\varepsilon)$. Then, we have
\begin{align*}
\tE(\widehat{f}_n,\varepsilon)  &=\widetilde\cR_{P}(\widehat{f}_n,\varepsilon) -\inf_{f\in \mathcal{H}^{\alpha}\cup \cN\cN(W,L,K)} \widetilde\cR_{P}(f,\varepsilon)
\\& =\widetilde\cR_{P}(\widehat{f}_n,\varepsilon) - \min\big\{\widetilde\cR_{P}(f^{\star},\varepsilon),\widetilde\cR_{P}(f_0,\varepsilon) \big\}
\\&=\max\{\cE_1,\cE_2\},
\end{align*}
where $\cE_1=\widetilde\cR_{P}(\widehat{f}_n,\varepsilon)- \widetilde\cR_{P}(f^{\star},\varepsilon)$ and $\cE_2=\widetilde\cR_{P}(\widehat{f}_n,\varepsilon)- \widetilde\cR_{P}(f_0,\varepsilon).$ For the   error $\cE_1$, we have the decomposition
\begin{align*}
\cE_1 &=\widetilde\cR_{P}(\widehat{f}_n,\varepsilon)-\widetilde\cR_{P_n}(\widehat{f}_n,\varepsilon)+\widetilde\cR_{P_n}(\widehat{f}_n,\varepsilon)-\widetilde\cR_{P_n}(f^{\star},\varepsilon)
\\&\quad
+\widetilde\cR_{P_n}(f^{\star},\varepsilon)-\widetilde\cR_{P}(f^{\star},\varepsilon)
\\&=I_1+I_2+I_3,
\end{align*}
where the errors are
$I_1  = \widetilde\cR_{P}(\widehat{f}_n,\varepsilon)-\widetilde\cR_{P_n}(\widehat{f}_n,\varepsilon)
$, $I_2=\widetilde\cR_{P_n}(\widehat{f}_n,\varepsilon)-\widetilde\cR_{P_n}(f^{\star},\varepsilon)
$,  and $I_3 =\widetilde\cR_{P_n}(f^{\star},\varepsilon)-\widetilde\cR_{P}(f^{\star},\varepsilon).$  For the error $\cE_2$, based on   $\widetilde\cR_{P_n}(\widehat{f}_n,\varepsilon)\leq \widetilde\cR_{P_n}(f_0,\varepsilon),$   we have the decomposition
\begin{align*}
\cE_2 &=\widetilde\cR_{P}(\widehat{f}_n,\varepsilon)-\widetilde\cR_{P_n}(\widehat{f}_n,\varepsilon)+\widetilde\cR_{P_n}(\widehat{f}_n,\varepsilon)-\widetilde\cR_{P_n}(f_0,\varepsilon)
\\&\quad
+\widetilde\cR_{P_n}(f_0,\varepsilon)-\widetilde\cR_{P}(f_0,\varepsilon)
\\&\leq I_1+I_4,
\end{align*}
where $I_4=\widetilde\cR_{P_n}(f_0,\varepsilon)-\widetilde\cR_{P}(f_0,\varepsilon)$.
We show the non-asymptotic upper bound for the adversarial excess risk by deriving N upper bound for each error term.

\subsubsection{%Error
Bounds for $I_1$ and $I_4$}
\label{sec:I1}
For any $f\in\cN\cN(W,L,K)$ and $\bz=(\bx,y)\in\cZ$,   define
$$
\tilde{\ell}(f,\bz)=\sup_{\bx^{\prime}\in B_{\varepsilon}(\bx)}\ell(f(\bx^{\prime}),y)
=\sup_{\bdelta\in B_{\varepsilon}(0)}\ell(f(\bx+\bdelta),y).
$$
Since $\widehat{f}_n$ and $f_0$ belong  to the   class $\cN\cN(W,L,K)$,  we have
\begin{align*}
I_1  =\widetilde\cR_{P}(\widehat{f}_n,\varepsilon)-\widetilde\cR_{P_n}(\widehat{f}_n,\varepsilon)
%\\&=\Ebb_{Z\sim P}\tilde{\ell}(\widehat{f}_n,Z)-\Ebb_{Z\sim P_n}\tilde{\ell}(\widehat{f}_n,Z)
\leq \sup_{f\in \cN\cN(W,L,K)}\big\{\Ebb_{Z\sim P}\tilde{\ell}(f,Z)-\Ebb_{Z\sim P_n}\tilde{\ell}(f,Z)\big\}.
\end{align*}
and
\begin{align*}
I_4 =\widetilde\cR_{P_n}(f_0,\varepsilon)-\widetilde\cR_{P}(f_0,\varepsilon)
%\\&=\Ebb_{Z\sim P}\tilde{\ell}(\widehat{f}_n,Z)-\Ebb_{Z\sim P_n}\tilde{\ell}(\widehat{f}_n,Z)
\leq \sup_{f\in \cN\cN(W,L,K)}\big\{\Ebb_{Z\sim P_n}\tilde{\ell}(f,Z)-\Ebb_{Z\sim P }\tilde{\ell}(f,Z)\big\}.
\end{align*}

Let the random vector $ \boldsymbol{\sigma}=(\sigma_1,\dots, \sigma_n)$ consist  of  i.i.d. the Rademacher variables that are independent from the data. The Rademacher variable  takes equal probability of being $1$ or $-1$. Denote the samples by $Z_{1:n}=\{Z_{i}\}_{i=1}^n$, with $Z_{i}=( X_i ,Y_i).$
Let  $Z_{i}^{\prime}=( X_i^{\prime} ,Y_i^{\prime}),i=1,\dots,n$, be generated i.i.d. from $P$ and be independent of  $Z_{1:n}$. The sample   $Z_{1:n}^{\prime}=\{Z_{i}^{\prime}\}_{i=1}^n$ is called as the ghost sample  of $Z_{1:n}$.  Then,
\begin{equation}\label{eq:Rade}
\begin{split}
&\sup_{f\in\cN\cN (W,L,K)}\big\{ \Ebb_{P}[\tilde{\ell}(f,Z)]  - \Ebb_{P_n}[\tilde{\ell}(f,Z)]\big\}
\\&=\sup_{f\in\cN\cN (W,L,K)}\Ebb_{\boldsymbol{\sigma}}\big\{ \Ebb_{P}[\tilde{\ell}(f,Z)]  - \Ebb_{P_n}[\tilde{\ell}(f,Z)]\big\}
\\&=\sup_{f\in\cN\cN (W,L,K)}\Ebb_{\boldsymbol{\sigma}}\big\{ \Ebb_{Z^{\prime}_{1:n}}[ \frac{1}{n}\sum_{i=1}^n\tilde{\ell}(f,Z_i^{\prime})]  - \frac{1}{n}\sum_{i=1}^n\tilde{\ell}(f,Z_i)\big\}
\\&=\sup_{f\in\cN\cN (W,L,K)}\Ebb_{Z^{\prime}_{1:n}}\big\{ \Ebb_{\boldsymbol{\sigma}}\big[ \frac{1}{n}\sum_{i=1}^n\tilde{\ell}(f,Z_i^{\prime})  -\tilde{\ell}(f,Z_i)\big]\big\}
\\&\leq\Ebb_{Z^{\prime}_{1:n}}\Ebb_{\boldsymbol{\sigma}}\big\{\sup_{f\in\cN\cN (W,L,K)} \big[ \frac{1}{n}\sum_{i=1}^n\tilde{\ell}(f,Z_i^{\prime})  -\tilde{\ell}(f,Z_i)\big]\big\}		
\\&=\Ebb_{Z^{\prime}_{1:n}}\Ebb_{\boldsymbol{\sigma}}\big\{\sup_{f\in\cN\cN (W,L,K)}  \frac{1}{n}\sum_{i=1}^n\sigma_i\big[\tilde{\ell}(f,Z_i^{\prime})  -\tilde{\ell}(f,Z_i)\big]\big\},
\end{split}
\end{equation}
where the last equality holds since $\tilde{\ell}(f,Z_i^{\prime})   -\tilde{\ell}(f,Z_i)$ are symmetric random variables, for which    they have the same distribution as  $\sigma_i(\tilde{\ell}(f,Z_i^{\prime})   -\tilde{\ell}(f,Z_i))$ \citep[Chapter 6.4]{vershynin2018}. Define the class of functions $\cL_n$  by
\[
\cL_n=\Big\{\tilde{\ell}(f,\bz):\cZ\mapsto \Rbb\mid f\in\cN\cN (W,L,K)\Big\}.
\]
For a given   set of samples $\bz_1,\dots, \bz_n$ from $\cZ$,   the empirical Rademacher complexity of class $\cL_n$ is defined by
\[
\widehat{\Re}_{n}(
\cL_n)=	\Ebb_{\boldsymbol{\sigma}}\Big\{\sup_{f\in\cN\cN (W,L,K) }  \frac{1}{n}\sum_{i=1}^n\sigma_i \tilde{\ell}(f,\bz_i)\Big\}.
\]
We analyze the  Rademacher complexity following the method motivated from \cite{mustafa22a}.
For a given $\tau\in(0,\varepsilon)$, let $C_{B_{\varepsilon}}(\tau)$ be a $(\tau, \|\cdot\|_{\infty})$-cover of $B_{\varepsilon}(0)$ with the smallest   cardinality $M_{\tau}$. Denote the elements of $C_{B_{\varepsilon}}(\tau)$  by $\bdelta_1,\dots, \bdelta_{M_{\tau}}$.
It follows by Lemma
SM1.3
%\ref{lem_coverball}
that
$
\log M_{\tau}\leq c d \log(  \varepsilon \tau^{-1})
$
for a constant $c.$
For any $\bz=(\bx,y)\in\cZ$,  the continuity of $\ell$ and $f$ imply that there exists $\bdelta^{\prime}\in B_{\varepsilon}(0)$ such that $\tilde{\ell}(f,\bz)=\ell(f(\bx+ \bdelta^{\prime}),y)$. Then,
\begin{align*}
\tilde{\ell}(f,\bz )-\max_{j}\ell(f(\bx+\bdelta_j),y) &= \ell(f(\bx+\bdelta^{\prime}),y)-\max_{j}\ell(f(\bx+\bdelta_j),y)
\\&\leq \min_j|\ell(f(\bx+\bdelta^{\prime}),y)-\ell(f(\bx+\bdelta_j),y)|
\\&\leq \Lip^1(\ell) \Lip(f)\min_j\|\bdelta^{\prime}-\bdelta_j\|_{\infty}
\\&\leq \Lip^1(\ell) \Lip(f)\tau.
\end{align*}
Therefore,  for any $f\in\cN\cN (W,L,K),$
\begin{align*}
&\frac{1}{n}\sum_{i=1}^n\sigma_i \tilde{\ell}(f,\bz_i)
\\&=\frac{1}{n}\sum_{i=1}^n\big\{\sigma_i \tilde{\ell}(f,\bz_i)-\sigma_i\max_{j}\ell(f(\bx_i +\bdelta_j),y_i)
+\sigma_i\max_{j}\ell(f(\bx_i +\bdelta_j),y_i)\big\}
\\&\leq \Lip^1(\ell)K\tau+\frac{1}{n}\sum_{i=1}^n\big\{\sigma_i\max_{j}\ell(f(\bx_i +\bdelta_j),y_i)\big\}.
\end{align*}
This leads to an upper bound of $\widehat{\Re}_{n}(\cL_n)$ as follows.
\begin{align*}
\widehat{\Re}_{n}(\cL_n)  \leq  \Ebb_{\boldsymbol{\sigma}}\Big\{\sup_{f\in\cN\cN (W,L,K) }  \frac{1}{n}\sum_{i=1}^n\sigma_i\max_{j}\ell(f(\bx_i +\bdelta_j),y_i)\Big\}
+\Lip^1(\ell) K\tau.
\end{align*}
Define the class
$$
\cL_{n,\tau}=\Big\{\max_{j}\ell(f(\bx  +\bdelta_j),y ):\cZ\mapsto \Rbb\mid f\in\cN\cN (W,L,K)\Big\}.
$$
Let $\cN(u,\cL_{n,\tau},L_\infty(P_n))$ denote the  covering number of the class  $\cL_{n,\tau}$ under the data dependent $L_\infty$ metric.  Define $S_{nM_{\tau}}=\{ \bx_i+\bdelta_j: i=1,\dots, n, j=1,\dots, M_{\tau}\}$.
For  the data set $S_{nM_{\tau}}$, let $\cN(u,\cN\cN(W,L,K),L_\infty(P_{nM_{\tau}}))$ denote the covering number of the class $\cN\cN(W,L,K)$ under the data dependent $L_\infty$ metric.
For any $f\in\cN\cN(W,L,K)$, there exists $f^{\prime}$ such that $\max_{i,j}|f(\bx_i+\bdelta_j)-f^{\prime}(\bx_i+\bdelta_j)|\leq u$, which leads to
\begin{align*}
&\max_i\big|\max_{j}\ell(f(\bx_i  +\bdelta_j),y_i)-\max_{j}\ell(f^{\prime}(\bx_i +\bdelta_j),y_i)\big|
\\&\leq \max_{i,j}|\ell(f(\bx_i  +\bdelta_j),y_i)-\ell(f^{\prime}(\bx_i +\bdelta_j),y_i)|  \\&\leq \Lip^1(\ell)	\max_{i,j}|f(\bx_i+\bdelta_j)-f^{\prime}(\bx_i+\bdelta_j)|
\\&\leq \Lip^1(\ell)u.
\end{align*}
Hence, we show
\begin{align*}
\cN(u,\cL_{n,\tau},L_\infty(P_n))
\leq \cN(u/\Lip^1(\ell),\cN\cN(W,L,K),L_\infty(P_{nM_{\tau}})).
\end{align*}
Suppose the functions in $\cN\cN(W,L,K)$ are uniformly bounded, otherwise they can be truncated. Define the uniform covering number of $\cN\cN(W,L,K)$ by
\begin{align*}
\cN_{\infty}(u,\cN\cN(W,L,K),n)
=\sup_{P_n}\cN(u,\cN\cN(W,L,K),L_\infty(P_{n})),
\end{align*}
where the supremum runs over all the data set of size $n$. Combining Lemmas
SM1.1
%\ref{lem_cover}
 and
 SM1.2%\ref{lem_psuedo}
, we derive
\[
\log \mathcal{N}_{\infty}(u, \cN\cN(W,L), n)\leq C_1 W^2 L^2 \log (W^2L) \log (u^{-1} n)
\]
for a constant $C_1.$ It follows
\begin{align*}
\log \cN(u,\cL_{n,\tau},L_\infty(P_n))
\leq C_2 W^2 L^2 \log (W^2L) \log (u^{-1} nM_{\tau})
\end{align*}
for a constant $C_2.$  Since the class $\cN\cN(W,L,K)$ is bounded and $\ell$ is continuous, there exists $B>0$ such that $\sup_{\bz\in\cZ}|\max_{j}\ell(f(\bx  +\bdelta_j),y)|\leq B$ for any $f\in \cN\cN(W,L,K)$. From Lemma
SM1.4
%\ref{lem:cover}
 and $
\log M_{\tau}\leq c d \log(  \varepsilon \tau^{-1})
$, we have
\begin{align*}
&\Ebb_{\boldsymbol{\sigma}}\Big\{\sup_{f\in\cN\cN (W,L,K) }  \frac{1}{n}\sum_{i=1}^n\sigma_i\max_{j}\ell(f(\bx_i +\bdelta_j),y_i)\Big\}
\\&\leq \inf_{\delta\geq 0}\Big\{
4\delta+12\int_{\delta}^{B}\sqrt{\frac{\log\cN(u,\cL_{n,\tau},L_\infty(P_n))}{n}}du
\Big\}
\\&\lesssim   \inf_{\delta\geq 0}\Big\{
\delta+ WL\sqrt{\log (W^2L)} n^{-1/2} \cdot \int_{\delta}^{B}\big[\sqrt{ \log (u^{-1})} +\sqrt{\log n}+\sqrt{ \log M_{\tau} }\;\big]du
\Big\}
\\&\lesssim   WL\sqrt{\log (W^2L)} n^{-1/2} \big\{\sqrt{\log  n } +\sqrt{ \log ( \varepsilon\tau^{-1}) } \big\}.
\end{align*}
Therefore,
%\begin{equation*}
%	I_1\lesssim  K\tau+WL\sqrt{\log (W^2L)} n^{-1/2} \big\{\sqrt{\log  n } +\sqrt{ \log ( \varepsilon\tau^{-1}) } \big\}.
%\end{equation*}
by selecting $\tau$ such that $\varepsilon\tau^{-1}=O(n),$ we show
\begin{align*}
I_1 \lesssim K\varepsilon n^{-1}+ WL\sqrt{\log (W^2L)} n^{-1/2} \sqrt{\log  n } .
\end{align*}
Following a similar procedure,  we have $
I_4\lesssim K\varepsilon n^{-1}+ WL\sqrt{\log (W^2L)} n^{-1/2} \sqrt{\log  n }$.

%
%
%The Rademacher complexity w.r.t. $\cN\cN (W,L,K)$ was investigated in \cite{jiao2022approximation} and   the results are restated in   Lemmas \ref{lem:SNN} and \ref{lem:com_SNN}. Combining Lemmas \ref{lem:SNN} and \ref{lem:com_SNN},  we have an upper bound of the Rademacher complexity as follows. For any $\bx_1,\dots, \bx_n\in [-B,B]^{d}$ with $B\geq 1$, let $S=\{(f(\bx_1),\dots, f(\bx_n)): f\in\cN\cN (W,L,K)\}\subseteq \Rbb^n$, then
%	\[
%	\Re_{n}(S)\leq \frac{2BK\sqrt{L+2+\log(d+1)}}{\sqrt{n}}.
%	\]
%It then follows by Talagrand’s contraction lemma  \ref{lem:Tala},
%\begin{align*}
% &\Ebb_{\boldsymbol{\sigma}}\Big\{\sup_{f\in\cN\cN (W,L,K)} \Big[ \frac{1}{n}\sum_{i=1}^n\sigma_i\ell(f(X_i^{\prime}),Y_i^{\prime})  -\sigma_i\ell(f(X_i),Y_i)\Big]\Big\}
%	\\&\leq \Ebb_{\boldsymbol{\sigma}}\Big\{\sup_{f\in\cN\cN (W,L,K)} \Big[ \frac{1}{n}\sum_{i=1}^n\sigma_i\ell(f(X_i^{\prime}),Y_i^{\prime})\Big]\Big\}  +\Ebb_{\boldsymbol{\sigma}}\Big\{\sup_{f\in\cN\cN (W,L,K)} \Big[ \frac{1}{n}\sum_{i=1}^n  \sigma_i\ell(f(X_i),Y_i)\Big]\Big\}
%	\\&\leq 4 K \Lip^1(\ell)\sqrt{L+2+\log(d+1)}n^{-1/2}.
%\end{align*}
%Therefore,   we derive
%\begin{align*}
%  I_1 &\leq     4 K \Lip^1(\ell)\sqrt{L+2+\log(d+1)}n^{-1/2}+2  (K+1)\Lip(\ell)\varepsilon
% \\&\lesssim   K  \sqrt{L+2+\log(d+1)}n^{-1/2}+ K \varepsilon.
%\end{align*}

\subsubsection{%Error
Bound for $I_2$}
\label{sec:I2}
Define  the approximation error   by
\begin{align*}
\cE\left(\cH^{\alpha},\cN\cN(W,L,K)\right)
=\sup_{f\in\cH^{\alpha}}\inf_{\phi\in \cN\cN(W,L,K)}\|f-\phi\|_{C([0,1]^d)},
\end{align*}
where $C([0,1]^d)$ is the space of continuous functions on $[0,1]^d$ equipped with the supremum norm. There exists $\bar{f}\in\cN\cN(W,L,K)$  approximating     the  target function $f^{\star}\in\cH^{\alpha}$ such that
\begin{equation*}
\|f^{\star}-\bar{f}\|_{C([0,1]^d)}=O(\cE\left(\cH^{\alpha},\cN\cN(W,L,K)\right)).
\end{equation*}
The difference between the empirical adversarial risks $\widetilde  \cR_{P_n}(f^{\star},\varepsilon)$ and $ \widetilde  \cR_{P_n}(\bar{f},\varepsilon)$   satisfies
\begin{align*}
\left|\widetilde  \cR_{P_n}(f^{\star},\varepsilon)-  \widetilde  \cR_{P_n}(\bar{f},\varepsilon)\right|
%	\\&= \left|\frac{1}{n}\sum_{i=1}^{n} \big[ \sup_{ X^{\prime}_i \in B_{\varepsilon}(X_i) }\ell(f^{\star} (X^{\prime}_i) ,Y_i) \big]-  \frac{1}{n}\sum_{i=1}^{n}\big[ \sup_{ X^{\prime}_i \in B_{\varepsilon}(X_i) }\ell(\bar{f} (X^{\prime}_i) ,Y_i)\big]\right|
&\leq\frac{1}{n}\sum_{i=1}^{n} \Big| \sup_{X^{\prime}_i \in B_{\varepsilon}(X_i) }\ell(f^{\star} (X^{\prime}_i) ,Y_i)-   \sup_{ X^{\prime}_i \in B_{\varepsilon}(X_i) }\ell(\bar{f} (X^{\prime}_i) ,Y_i)\Big|
\\&\leq \frac{1}{n}\sum_{i=1}^{n}\sup_{ X^{\prime}_i \in B_{\varepsilon}(X_i) } \Big| \ell(f^{\star} (X^{\prime}_i) ,Y_i)-\ell(\bar{f} (X^{\prime}_i) ,Y_i) \Big|
\\&\leq  \Lip^1(\ell)\cdot\|f^{\star}-\bar{f}\|_{C([0,1]^d)}.
\end{align*}
Since  $\widehat{f}_n$ minimizes the empirical adversarial risk over the class $\cN\cN(W,L,K)$,   then
\begin{align*}
I_2= \widetilde\cR_{P_n}(\widehat{f}_n,\varepsilon)-\widetilde\cR_{P_n}(f^{\star},\varepsilon)
& = \widetilde\cR_{P_n}(\widehat{f}_n,\varepsilon)-\widetilde\cR_{P_n}(\bar{f},\varepsilon)+\widetilde\cR_{P_n}(\bar{f},\varepsilon)-\widetilde\cR_{P_n}(f^{\star},\varepsilon)
\\&\leq \widetilde\cR_{P_n}(\bar{f},\varepsilon)-\widetilde\cR_{P_n}(f^{\star},\varepsilon)
\\&\leq  \Lip^1(\ell)\cdot\|f^{\star}-\bar{f}\|_{C([0,1]^d)}.
\end{align*}

The approximation error $\cE\left(\cH^{\alpha},\cN\cN(W,L,K)\right)$ is investigated  by \cite{jiao2022approximation} and the result is given as follows.
\begin{lemma}[\cite{jiao2022approximation} Theorem 3.2]\label{lem:appro}
	Let $\gamma=\lceil \log_2 (d+r)\rceil$.	There exists $c>0$ such that for any $W\geq c(K/\log^{\gamma}K)^{(2d+\alpha)/(2d+2)}$ and $L\geq 4\gamma+2$,
	\[
	\cE\left(\cH^{\alpha},\cN\cN(W,L,K)\right)\lesssim (K/\log^{\gamma}K)^{-\alpha/(d+1)}.
	\]
\end{lemma}
Therefore,  we derive
\begin{align*}
I_2
%=\widetilde\cR_{P_n}(\widehat{f}_n,\varepsilon)-\widetilde\cR_{P_n}(f^{\star},\varepsilon)
\lesssim (K/\log^{\gamma}K)^{-\alpha/(d+1)}.
\end{align*}

\subsubsection{Bound for $I_3$}
\label{sec:I3}
For any $f\in\cH^{\alpha}$ and $\bz=(\bx,y)\in\cZ$, we define $\tilde{\ell}(f,\bz)=\sup_{\bx^{\prime}\in B_{\varepsilon}(\bx)}\ell(f(\bx^{\prime}),y)$. The error $I_3$ can be  upper bounded by
\begin{align*}
I_3   =\widetilde\cR_{P_n}(f^{\star},\varepsilon)-\widetilde\cR_{P}(f^{\star},\varepsilon)
%\\&= \Ebb_{Z\sim P_n}\tilde{\ell}(f^{\star},Z)-\Ebb_{Z\sim P}\tilde{\ell}(f^{\star},Z)
\leq \sup_{f\in\cH^{\alpha}}\big\{\Ebb_{Z\sim P_n}\tilde{\ell}(f,Z)-\Ebb_{Z\sim P}\tilde{\ell}(f,Z)\big\}.
\end{align*}
Define the class $\cL^{\alpha}=\{\tilde{\ell}(f,\bz):\cZ\mapsto \Rbb \mid f\in\cH^{\alpha}\}$.
Let $\boldsymbol{\sigma}=(\sigma_1,\dots, \sigma_n)$  consist   of  i.i.d. the Rademacher variables and be  independent from the data.  For any samples $\bz_1,\dots,\bz_n$ from $\cZ$, we   denote the empirical Rademacher complexity of the class $\cL^{\alpha}$ by
\[
\widehat{\Re}_{n}(
\cL^{\alpha})=	\Ebb_{\boldsymbol{\sigma}}\Big\{\sup_{f\in\cH^{\alpha} }  \frac{1}{n}\sum_{i=1}^n\sigma_i \tilde{\ell}(f,\bz_i)\Big\}.
\]
From $\sup_{f\in\cH^{\alpha}}\|f\|_{\infty}\leq 1$, there exists a constant  $B$ such that  $\sup_{\bz\in\cZ}|\tilde{\ell}(f,\bz)|\leq B$ for any $f\in\cH^{\alpha}$. In addition, we have
\[
\log\mathcal{N}\big(u,  \cL^{\alpha} ,\|\cdot\|_{\infty}\big)\leq 	\log\mathcal{N}\big(u/\Lip^1(\ell),  \cH^{\alpha} ,\|\cdot\|_{\infty}\big).
\]
This is because  that for any $f$,  $\tilde{f}\in\cH^{\alpha} $ satisfying $\|f-\tilde{f}\|_{\infty}\leq u/\Lip^1(\ell)$, it follows
\begin{align*}
\big|\tilde{\ell}(f,\bz)-\tilde{\ell}(\tilde{f},\bz)\big|
&=  \big|\sup_{\bx^{\prime}\in B_{\varepsilon}(\bx)}\ell(f(\bx^{\prime}),y)-\sup_{\bx^{\prime}\in B_{\varepsilon}(\bx)}\ell(\tilde{f}(\bx^{\prime}),y)\big|
%\\&\leq  \sup_{\bx^{\prime}\in B_{\varepsilon}(\bx)}\big|\ell(f(\bx^{\prime}),y)- \ell(\tilde{f}(\bx^{\prime}),y)\big|
\\&\leq \Lip^1(\ell)\sup_{\bx^{\prime}\in B_{\varepsilon}(\bx)}\big|f(\bx^{\prime})- \tilde{f}(\bx^{\prime})\big|
\\&\leq u.
\end{align*}
From  \cite{kolmogorov1959varepsilon}, $ 	\log\mathcal{N}\big(u,  \cH^{\alpha} ,\|\cdot\|_{\infty}\big)  \lesssim u^{-d/\alpha} $ holds. Hence,
\[
\log\mathcal{N}\big(u,  \cL^{\alpha},L_2(P_n)\big)\leq 	\log\mathcal{N}\big(u,  \cL^{\alpha} ,\|\cdot\|_{\infty}\big) \lesssim u^{-d/\alpha},
\]
where $L_2(P_n)$ denotes the   $L_2$ metric generated by the samples.
It follows by Lemma
SM1.4
%\ref{lem:cover}
 that
\begin{align*}
\widehat{\Re}_{n}(\cL^{\alpha})\lesssim \inf_{\delta\geq 0}\Big\{
4\delta+12\int_{\delta}^{1}\sqrt{\frac{\log\mathcal{N}(u,\cL^{\alpha},\|\cdot\|_{\infty})}{n}}du
\Big\}.
\end{align*}
Let $\gamma=d/(2\alpha)$.
Thus, we show
\begin{align*}
\widehat{\Re}_{n}(\cL^{\alpha}) \lesssim  \inf _{\delta\geq 0}\Big(\delta+ n^{-1/2}\int_{\delta}^{1} u^{-\gamma}du \Big).
\end{align*}
When $\gamma>1,$ by taking $\delta =n^{-\alpha/d}$,  one has
\begin{align*}
\widehat{\Re}_{n}(\cL^{\alpha}) \lesssim  \inf _{\delta\geq 0}\Big(\delta+ (\gamma-1)^{-1}n^{-1/2}(\delta^{1-\gamma}-1) \Big)\lesssim n^{-\alpha/d}.
\end{align*}
When $\gamma=1,$  by taking $\delta=n^{-1/2}$, one has
\begin{align*}
\widehat{\Re}_{n}(\cL^{\alpha})\lesssim  \inf _{\delta\geq 0}\Big(\delta- n^{-1/2}\log \delta  \Big)\lesssim n^{-1/2}\log n.
\end{align*}
When $\gamma<1,$ one has
\begin{align*}
\widehat{\Re}_{n}(\cL^{\alpha})\lesssim  \inf _{\delta\geq 0}\Big(\delta+(1-\gamma)^{-1} n^{-1/2}(1-\delta^{1-\gamma}) \Big)\lesssim n^{-1/2}.
\end{align*}
Combining these cases together, we derive
\begin{align*}
\widehat{\Re}_{n}(\cL^{\alpha}) \lesssim  n^{-\min\{1/2,\alpha/d\}}\log^{c(\alpha,d)}n,
\end{align*}
where $c(\alpha,d)=1$ if $ d=2\alpha$, and $c(\alpha,d)=0,$ otherwise. Following    a similar analysis method in \eqref{eq:Rade}, we derive
\begin{equation*}
\begin{split}
\sup_{f\in\cH^{\alpha}}\big\{\Ebb_{Z\sim P_n}\tilde{\ell}(f,Z)-\Ebb_{Z\sim P}\tilde{\ell}(f,Z)\big\}
\lesssim n^{-\min\{1/2,\alpha/d\}}\log^{c(\alpha,d)}n.
\end{split}
\end{equation*}
Consequently,   it follows
\begin{equation*}
I_3  \lesssim  n^{-\min\{1/2,\alpha/d\}}\log^{c(\alpha,d)}n.
\end{equation*}

	Let $\gamma=\lceil \log_2 (d+r)\rceil$. Combining the results from Sections \ref{sec:I1}--\ref{sec:I3}, we show for any $W\geq c(K/\log^{\gamma}K)^{(2d+\alpha)/(2d+2)}$ and $L\geq 4\gamma+2$,	
	\begin{equation}\label{eq_11}
	\begin{split}
	\tE(\widehat{f}_n,\varepsilon) &=\max\{\cE_1,\cE_2\}
	\\&\leq  \max\{I_1+  I_2+  I_3,I_1+  I_4\}
	\\& \lesssim  K\varepsilon n^{-1}+ WL\sqrt{\log (W^2L)} n^{-1/2} \sqrt{\log  n }
	\\&\quad +( K/\log^{\gamma}K  )^{- \alpha/(d+1)}+n^{-\min\{1/2, \alpha/d  \}}\log^{c(\alpha,d)}n,
	\end{split}
	\end{equation}
	where  $c(\alpha,d)=1$ when $ d=2\alpha$, and $c(\alpha,d)=0,$ otherwise.
	By selecting $K\asymp n^{(d+1)/(2d+3\alpha)}$, and $WL\asymp n^{(2d+\alpha)/(4d+6\alpha)}$, we have
	\begin{align*}
	&WL\sqrt{\log (W^2L)} n^{-1/2} \sqrt{\log  n }
	+( K/\log^{\gamma}K  )^{- \alpha/(d+1)}
\\&	\lesssim  n^{-\alpha/(2d+3\alpha)}\log n^{\max\{1,\gamma\alpha/(d+1)\}}.
	\end{align*}
	This leads to
	\[
	\tE(\widehat{f}_n,\varepsilon)
	\lesssim  K\varepsilon n^{-1}  +n^{-\alpha/(2d+3\alpha)}\log n^{\xi},
	\]
	where $\xi=\max\{1,\gamma\alpha/(d+1)\}$. Hence,  the result is proved.

\subsection{Proof of Corollary 3.2}
\begin{proof}[Proof of Corollary 3.2]
	From Lemma 2.4 and  Theorem 3.1, we have
	\begin{align*}
	\cR_{P}(\widehat{f}_n)-\widetilde\cR_{P}(f^{\star},\varepsilon) \leq\widetilde\cR_{P}(\widehat{f}_n,\varepsilon) -\widetilde\cR_{P}(f^{\star},\varepsilon)  \lesssim r_n+K\varepsilon n^{-1},
	\end{align*}
	where 	  $r_n=n^{-\alpha/(2d+3\alpha)}\log n^{\max\{1,\gamma\alpha/(d+1)\}}$. In addition, we have
	\begin{align*}
	\inf_{f\in\cH^{\alpha}}\widetilde\cR_{P}(f,\varepsilon)\leq	\inf_{f\in\cH^{\alpha}} \cR_{P}(f)+\Lip^1(\ell) \varepsilon.
	\end{align*}
	Therefore, it follows
	\begin{align*}
	\cE(\widehat{f}_n)=
	\cR_{P}(\widehat{f}_n)-\inf_{f\in\cH^{\alpha}} \cR_{P}(f)\lesssim r_n+\varepsilon,
	\end{align*}
	which  completes the proof.
\end{proof}

\subsection{Proof of Corollary 3.3}
\begin{proof}[Proof of Corollary 3.3]
	From Lemma 2.3,   there exists $P^{\star}$ such that $\widetilde \cR_{P}(f,\varepsilon)= \cR_{P^{\star}}(f)$, where $P^{\star}$ is  dependent on $f$ and satisfies $W_\infty(P^{\star},P) \leq \varepsilon$. Since $W_p(P,Q) \leq  W_q(P,Q) $ for any $1\leq p\leq q\leq \infty,$  we have
	\begin{align*}
	\widetilde \cR_{P}(f,\varepsilon)\leq\sup_{Q:W_{\infty}(Q,P)\leq \varepsilon}\cR_{Q}(f)\leq \sup_{Q:W_{p}(Q,P)\leq \varepsilon}\cR_{Q}(f),
	\end{align*}
	for any $1\leq p \leq \infty.$ It implies
	\[
	\inf_{f\in \mathcal{H}^{\alpha}\cup \cN\cN(W,L,K)}\widetilde \cR_{P}(f,\varepsilon)\leq   \inf_{f\in \mathcal{H}^{\alpha}\cup \cN\cN(W,L,K)}\sup_{Q:W_{1}(Q,P)\leq \varepsilon}\cR_{Q}(f).
	\]
	For a function $\psi:\cZ\mapsto \Rbb,$ the Lipschitz constant is defined by
	\[
	\Lip(\psi)=\sup_{\bz_1\ne \bz_2}\frac{|\psi(\bz_1)-\psi(\bz_2)|}{d_{\cZ} (\bz_1, \bz_2 )}.
	\]
	Let the 1-Lipschitz    function class  be defined by
	\[ \Psi_{\Lip}^1=\big\{\psi:\cZ\mapsto \Rbb,\; \Lip(\psi)\leq 1\big\}.\]
	The $1$-th Wasserstein distance $W_1$ has an useful duality formula
	\begin{align*}
	W_1(Q,P) = \sup_{\psi\in \Psi_{\Lip}^1}\Big\{
	\Ebb_{Z\sim Q}[\psi(Z)]-\Ebb_{Z\sim P}[\psi(Z)]	\Big\}.
	\end{align*}
	It follows that  for any $Q $  satisfying $W_1(Q,P)\leq \varepsilon$, we have
	\begin{align*}
	\cR_{Q}(f)-\cR_{P^{\star}}(f)
	=\Ebb_{Z\sim Q}[\ell(f(X),Y)]-\Ebb_{Z\sim P^{\star}}[\ell(f(X),Y)]
	\leq 2\Lip(\ell)(\Lip(f)+1) \varepsilon.
	\end{align*}
	It  shows
	\[
	\sup_{Q:W_1(Q,P)\leq \varepsilon}\cR_{Q}(\widehat{f}_n)\leq 	\widetilde \cR_{P}(\widehat{f}_n,\varepsilon)+2\Lip(\ell)(K+1) \varepsilon.
	\]
	Let $r_n=n^{-\alpha/(2d+3\alpha)}\log n^{\max\{1,\gamma\alpha/(d+1)\}}$. From Theorem 3.1,  we have
	\begin{align*}
	\sup_{Q:W_1(Q,P)\leq \varepsilon}\cR_{Q}(\widehat{f}_n)-\inf_{f\in \mathcal{H}^{\alpha}\cup \cN\cN(W,L,K)}\sup_{Q:W_{1}(Q,P)
		\leq \varepsilon}\cR_{Q}(f) \leq r_n+2\Lip(\ell)(K+1) \varepsilon,
	\end{align*}
	which completes the proof.
\end{proof}

\subsection{Proof of Corollary 4.3}
\begin{proof}[Proof of Corollary 4.3]
	Since Assumptions 4.1 and 4.2 hold for the margin loss function $\phi$, based on Lemma \ref{lem_ineqaulity}, we have	
	\begin{align*}
	\widetilde\cR_{P}(\widehat{f}_n,\varepsilon) - \inf_{f\in\cH^{\alpha}} \widetilde\cR_{P}(f,\varepsilon)  \geq  a_1  (\widetilde\cR_{\operatorname{class},P}(\widehat{f}_n,\varepsilon) - \cR_{\operatorname{class},P}^{\star})  - \inf_{f\in\cH^{\alpha}} \widetilde\cR_{P}(f,\varepsilon) + \cR_{P}^{\star}.
	\end{align*}
	for a constant $a_1>0.$ 	Based on Lemma 2.4, we have
	\[
	\inf_{f\in\cH^{\alpha}}	\widetilde \cR_{P}(f,\varepsilon)\leq\inf_{f\in\cH^{\alpha}} \cR_{P}(f)+\Lip^1(\ell) \varepsilon.
	\]
	Then,
	\begin{align*}
	\widetilde\cR_{\operatorname{class},P}(\widehat{f}_n,\varepsilon) - \widetilde\cR_{\operatorname{class},P}^{\star}(\varepsilon)\lesssim  \widetilde\cR_{P}(\widehat{f}_n,\varepsilon) - \inf_{f\in\cH^{\alpha}} \widetilde\cR_{P}(f,\varepsilon) +\varepsilon.
	\end{align*}
	Therefore, the result follows from Theorem 3.1.
\end{proof}

\subsection{Proof of Corollary 4.4}
\begin{proof}[Proof of Corollary 4.4]
	Based on Lemma 2.4, we have
	\[
	\inf_{f\in\cH^{\alpha}}	\widetilde \cR_{P}(f,\varepsilon)\leq\inf_{f\in\cH^{\alpha}} \cR_{P}(f)+\Lip^1(\ell) \varepsilon.
	\]
	For any measurable function $f$, we have
	\begin{align*}
	a  ( \cR_{\operatorname{class},P}(f) -  \cR_{\operatorname{class},P}^{\star} )  &\leq 	 \cR_{P}(f) -   \cR_{P}^{\star}
	\\& \leq \widetilde{\cR}_{P}(f,\varepsilon)-\inf_{f\in\cH^{\alpha}} \widetilde\cR_{P}(f,\varepsilon)+\inf_{f\in\cH^{\alpha}} \widetilde\cR_{P}(f,\varepsilon)-\cR_{P}^{\star}
	\\&\leq  \widetilde{\cR}_{P}(f,\varepsilon)-\inf_{f\in\cH^{\alpha}} \widetilde\cR_{P}(f,\varepsilon) +\Lip^1(\ell) \varepsilon.
	\end{align*}
	Therefore, the result follows from Theorem 3.1.
\end{proof}

\subsection{Proof of Theorem 4.5}
\begin{proof}[Proof of Theorem 4.5]
	From Lemma 2.4, for any $\bar{f}\in\cN\cN(W,L,K)$, we have
	\begin{equation}\label{eq_square}
	\begin{split}
	\widetilde\cR_{P}(\widehat{f}_n^{ls},\varepsilon)-\widetilde\cR_{P}(f^{\star},\varepsilon)
	&=\widetilde\cR_{P}(\widehat{f}_n^{ls},\varepsilon)-\widetilde\cR_{P_n}(\widehat{f}_n^{ls},\varepsilon)+\widetilde\cR_{P_n}(\widehat{f}_n^{ls},\varepsilon)
	-\widetilde\cR_{P_n}(\bar{f} ,\varepsilon)
	\\&\quad +\widetilde\cR_{P_n}(\bar{f} ,\varepsilon)-\widetilde\cR_{P}(\bar{f} ,\varepsilon)
	+\widetilde\cR_{P}(\bar{f},\varepsilon) -\widetilde\cR_{P}(f^{\star},\varepsilon)
	\\&\leq \widetilde\cR_{P}(\widehat{f}_n^{ls},\varepsilon)-\widetilde\cR_{P_n}(\widehat{f}_n^{ls},\varepsilon)+\widetilde\cR_{P_n}(\bar{f} ,\varepsilon)-\widetilde\cR_{P}(\bar{f} ,\varepsilon)
	\\&\quad+\widetilde\cR_{P}(\bar{f},\varepsilon) -\widetilde\cR_{P}(f^{\star},\varepsilon)
	\\&\leq 2A  +\Lip^1(\ell)K\varepsilon  + \cR_{P}(\bar{f})- \cR_{P}(f^{\star})
	\\&\leq 2A 	+\Lip^1(\ell)K\varepsilon+\|\bar{f}-f_0\|^2_{C([0,1]^d)},
	\end{split}
	\end{equation}
	where $A=\sup_{f\in \cN\cN(W,L,K)}|\widetilde\cR_{P}(f,\varepsilon)-\widetilde\cR_{P_n}(f,\varepsilon)|$ and the last inequality holds because
	\begin{align*}
	\cR_{P}(\bar{f})- \cR_{P}(f^{\star})
	\leq\Ebb[(\bar{f}(X)-f_0(X))^2]
	\leq \|\bar{f}-f_0\|^2_{C([0,1]^d)}.
	\end{align*}
	Let $f_0\in\argmin_{f\in    \cN\cN(W,L,K)} \widetilde\cR_{P}(f,\varepsilon)$. We have	
	\begin{align*}
	\widetilde\cR_{P}(\widehat{f}_n^{ls},\varepsilon)- \widetilde\cR_{P}(f_0,\varepsilon) &=\widetilde\cR_{P}(\widehat{f}_n^{ls},\varepsilon)-\widetilde\cR_{P_n}(\widehat{f}_n^{ls},\varepsilon)+\widetilde\cR_{P_n}(\widehat{f}_n^{ls},\varepsilon)-\widetilde\cR_{P_n}(f_0,\varepsilon)
	\\&\quad
	+\widetilde\cR_{P_n}(f_0,\varepsilon)-\widetilde\cR_{P}(f_0,\varepsilon)
	\\&\leq 2A.
	\end{align*}
	Based on Section
	A.1
%	\ref{sec:I1}
	,	the error bound for $A$ can be derived.   Section
	A.2
	%\ref{sec:I2}
	shows that there exists $\bar{f}_0\in\cN\cN(W,L,K)$ such that $\|\bar{f}_0-f_0\|_{C([0,1]^d)}\lesssim (K/\log^{\gamma}K)^{- \alpha/(d+1)}$. Since the inequality \eqref{eq_square} holds for any $\bar{f}\in\cN\cN(W,L,K)$, by selecting $\bar{f}=\bar{f}_0$, we derive
	\begin{align*}
	&\widetilde\cR_{P}(\widehat{f}_n^{ls},\varepsilon)-\widetilde\cR_{P}(f^{\star},\varepsilon)
	\\&\lesssim K\varepsilon n^{-1}+ WL\sqrt{\log (W^2L)} n^{-1/2} \sqrt{\log  n }
	+K\varepsilon+ (K/\log^{\gamma}K)^{-2\alpha/(d+1)}
	\\&\lesssim   WL\sqrt{\log (W^2L)} n^{-1/2} \sqrt{\log  n }  + (K/\log^{\gamma}K)^{-2\alpha/(d+1)}+K\varepsilon.
	\end{align*}
	By selecting $$K\asymp n^{(d+1)/(2d+5\alpha)}\text{ and }WL\asymp n^{(2d+\alpha)/(4d+10\alpha)},$$  it leads to
	\begin{align*}
	\tE(\widehat{f}_n^{ls},\varepsilon)&=\max\{ \widetilde\cR_{P}(\widehat{f}_n^{ls},\varepsilon)-\widetilde\cR_{P}(f^{\star},\varepsilon),	\widetilde\cR_{P}(\widehat{f}_n^{ls},\varepsilon)- \widetilde\cR_{P}(f_0,\varepsilon)\}
	\\&	\lesssim  n^{-2\alpha/(2d+5\alpha)}\log n^{\lambda}+n^{(d+1)/(2d+5\alpha)}\varepsilon,
	\end{align*}
	where $\lambda=\max\{1,2\gamma\alpha/(d+1)\}$. Hence, the proof completes.
\end{proof}

\subsection{Proof of Corollary 4.6}
\begin{proof}[Proof of Corollary 4.6]
	Based on Lemma 2.4, we have
	\[
	\cR_{P}(f)-\inf_{f\in\cH^{\alpha}} \cR_{P}(f)\leq \widetilde \cR_{P}(f,\varepsilon)-	\inf_{f\in\cH^{\alpha}}	\widetilde \cR_{P}(f,\varepsilon)+\Lip^1(\ell) \varepsilon.
	\]
	Since $ \cR_{P}(f)=\Ebb[(f(X)-f_0(X))^2]+\Ebb[\eta^2]$, then $ \inf_{f\in\cH^{\alpha}} \cR_{P}(f)=\Ebb[\eta^2]. $ It follows
	\begin{align*}
	\|\widehat{f}_n^{ls}-f_0\|_2^2=  \Ebb[(\widehat{f}_n^{ls}(X)-f_0(X))^2]&\leq \widetilde \cR_{P}(\widehat{f}_n^{ls},\varepsilon)-	\inf_{f\in\cH^{\alpha}}	\widetilde \cR_{P}(f,\varepsilon)+\Lip^1(\ell) \varepsilon
	\\&\leq 	\tE(\widehat{f}_n^{ls},\varepsilon)+\Lip^1(\ell) \varepsilon.
	\end{align*}
	Based on Theorem 4.5, the proof completes.
\end{proof}

\bibliographystyle{imsart-nameyear}

\bibliography{ad}
\end{document}